\documentclass{article}

 \usepackage[preprint]{neurips_2026}

\usepackage[utf8]{inputenc} %
\usepackage[T1]{fontenc}    %
\usepackage{hyperref}       %
\usepackage{url}            %
\usepackage{booktabs}       %
\usepackage{amsfonts}       %
\usepackage{nicefrac}       %
\usepackage{microtype}      %
\usepackage{xcolor}         %
\usepackage{amsmath}
\usepackage{amsthm}
\usepackage{enumitem}
\usepackage{amssymb}
\usepackage{xspace}

\usepackage[ruled,linesnumbered,vlined]{algorithm2e}
\SetKwComment{Comment}{\# }{}
\SetKwFunction{DefineFeatures}{DefineFeatures}
\SetKwFunction{ExtractFeatures}{ExtractFeatures}
\SetKwFunction{CrossValidate}{CrossValidate}
\SetKwFunction{FitGP}{FitGP}
\SetKwFunction{OptimizeAcq}{OptimizeAcquisition}
\SetKwFunction{GenRefine}{GenerateWithRefinement}
\SetKwFunction{FeatureRefine}{FeatureGuidedRefine}
\SetKwFunction{StratSample}{StratifiedSubsample}
\SetKw{ParFor}{parallel for}
\SetKw{Break}{break}
\SetKw{KwReturn}{return}

\usepackage[most]{tcolorbox}

\newtcolorbox{promptbox}[1][]{%
  colback=gray!5,
  colframe=black!60,
  fonttitle=\bfseries,
  enhanced,
  left=4pt, right=4pt, top=2pt, bottom=2pt,
  #1
}

\newcommand{\tvar}[1]{\textcolor{violet}{\texttt{\{#1\}}}}

\usepackage{booktabs}
\usepackage{tabularx}
\usepackage{multirow}

\usepackage{graphicx}
\usepackage{subcaption}

\usepackage[disable,textsize=tiny]{todonotes}

\newcommand{\method}{ReElicit\xspace} 

\newtheorem{assumption}{Assumption}
\newtheorem{lemma}{Lemma}
\newtheorem{theorem}{Theorem}
\newtheorem{definition}{Definition}
\title{Embedding by Elicitation: Dynamic Representations for Bayesian Optimization of System Prompts}

\author{%
  Zhiyuan Jerry Lin \\
  Meta\\
  \texttt{zylin@meta.com} \\
  \And
  Benjamin Letham \\
  Meta\\
  \texttt{bletham@meta.com} \\
  \And
  Samuel Dooley \\
  Meta\\
  \texttt{dooley@meta.com} \\
  \AND
  Maximilian Balandat \\
  Meta\\
  \texttt{balandat@meta.com} \\
  \And
  Eytan Bakshy \\
  Meta\\
  \texttt{ebakshy@meta.com}
}

\begin{document}

\maketitle

\begin{abstract}
System prompts are a central control mechanism in modern AI systems, shaping behavior across conversations, tasks, and user populations.
Yet they are difficult to tune when feedback is available only as aggregate metrics rather than per-example labels, failures, or critiques.
We study this aggregate feedback setting as sample-constrained black-box optimization over discrete, variable-length text.
We introduce \method, a Bayesian optimization framework based on \emph{embedding by elicitation}.
Given a task description, previously evaluated prompts, and scalar scores, an LLM elicits a compact, interpretable feature space and maps prompts into it.
Leveraging a probabilistic Gaussian process surrogate, an acquisition function then selects target feature vectors, which the LLM realizes and refines into deployable system prompts.
Re-eliciting the feature space as new evaluations arrive lets the representation adapt to the observed prompt-score history.
We evaluate the setting using offline benchmark accuracy as a controlled aggregate proxy: the optimizer observes one scalar score per prompt and no per-example labels, errors, or critiques.
Across ten system prompt optimization tasks with a 30 total evaluation budget, \method achieves the strongest aggregate performance profile among representative aggregate-only prompt-optimization baselines.
These results suggest that LLMs can serve as adaptive semantic representation builders, not only prompt generators, for Bayesian optimization over natural-language artifacts.
\end{abstract}

\section{Introduction}
The \emph{system prompt} is a central control mechanism in modern AI systems.
It shapes response style, guardrails, and operational policies that persist across conversations and tasks. As a result, small prompt changes can affect many downstream interactions.
Despite this importance, system prompts are still often crafted manually using developer intuition and limited offline evaluation.

Recent work on automatic prompt optimization (APO) seeks to tune prompts automatically to maximize a target objective~\citep{ramnath2025systematic}.
Many APO methods assume granular task feedback: a candidate prompt is evaluated on labeled examples, and the optimizer can inspect per-example successes, failures, traces, or critiques.
That interface is powerful, but it differs from many deployment-facing system prompt optimization settings where outcomes are delayed, population-level, or meaningful only after repeated interactions.
Examples include long-horizon task-completion rates, safety-incident rates, user satisfaction, retention, and escalation rates.
Although such metrics aggregate many individual interactions, the optimizer may observe only a top-line scalar score for each deployed prompt variant that is not separable into individual interactions.

In this regime, prompt optimization is no longer a supervised prompt-revision problem over labeled examples.
It is a sample-constrained black-box optimization problem over natural language.
Directly asking an LLM to propose better prompts is a natural baseline, but such search does not explicitly model uncertainty or provide a principled exploration--exploitation tradeoff.

Bayesian optimization (BO) is a natural tool for expensive scalar-feedback objectives.
It fits a probabilistic surrogate to past evaluations and uses an acquisition function to balance exploration and exploitation.
BO is widely used for hyperparameter tuning and machine-learning system design~\citep{shahriari2015taking,balandat2020botorch}, A/B testing~\citep{olson2025ax,letham2019constrained}, and other aggregate feedback LLM settings such as data mix optimization~\citep{yen2025data}.

The obstacle in our setting is representation and realization.
BO typically operates in a fixed low-dimensional Euclidean domain, whereas system prompts are discrete, variable-length, and semantically structured natural-language objects.
Structured-input BO can use hand-crafted kernels over discrete objects~\citep{oh2019combinatorial,moss2020boss,griffiths2023gauche} or learned latent representations~\citep{gomez2018automatic,deshwal2021combining,maus2022local}, but these tools do not directly give a full prompt-optimization loop.
Kernel methods still require searching over text by enumeration or sampling, while learned latent spaces usually require auxiliary data or task-specific encoder--decoder training.
We therefore need a compact space that supports surrogate modeling and acquisition optimization, together with a way to map optimized points in that space back to deployable system prompts.

Our approach is to use the LLM itself as a semantic representation builder.
Given a task description, evaluated prompts, and scalar scores, the LLM proposes a small set of performance-relevant feature axes and maps prompts into $[0,1]^{d_t}$.
The premise is not that prompts are intrinsically simple, but that the task-relevant variation observed under a specific objective may concentrate along a few semantically meaningful axes.
Useful axes may capture answer-format control, calibrated uncertainty, explicit reasoning structure, evidence use, or task-specific distinctions such as numerical consistency, ambiguity resolution, symbolic validity, or pragmatic intent.
These are not surface features such as length or token overlap; they are semantic directions along which prompts can differ.

This representation gives BO a compact continuous space for surrogate modeling and acquisition optimization.
The LLM then realizes BO-selected feature targets as deployable prompts and refines them using feature-gap feedback.
Re-eliciting the feature space as new evaluations arrive allows the representation to adapt to evidence about which prompt properties distinguish high- and low-performing candidates.

Our key contributions are:
\begin{itemize}[leftmargin=*]
    \item We cast aggregate feedback system prompt tuning as a black-box optimization problem, distinct from prompt optimization settings that rely on per-example labels, error traces, or textual critiques.

    \item We introduce \method, a Bayesian optimization framework based on \emph{embedding by elicitation}: an LLM elicits a semantic feature space from prompt-score history, BO selects target feature vectors in this space, and the LLM realizes and refines those targets as natural language system prompts.

    \item We provide a reachability analysis showing how representation error affects optimization in an elicited embedding: under an oracle smooth semantic embedding assumption, near-optimality for the approximated objective transfers to a bounded true prompt-quality gap.

    \item We evaluate \method on ten benchmark system prompt optimization tasks under a shared 30-evaluation budget and an aggregate-only feedback interface.
    \method achieves the strongest overall performance profile among representative aggregate-only APO baselines, and our diagnostics and ablations analyze feature stability, surrogate fit, refinement quality, and component contributions.
\end{itemize}

\section{Related Work}
\label{sec:related_work}

\paragraph{Automatic prompt optimization.}
Automatic prompt optimization searches directly over natural-language instructions~\citep{ramnath2025systematic}.
Many APO methods use instance-level feedback, such as per-example labels, textual critiques, error traces, or reflections on successes and failures; examples include ProTeGi, TextGrad, and GEPA~\citep{pryzant2023automatic,yuksekgonul2024textgrad,agrawal2025gepa}.
Other methods are closer to aggregate black-box search: APE samples prompts from an LLM~\citep{zhou2022large}, OPRO conditions generation on previous solutions and scores~\citep{yang2023large}, PromptBreeder uses evolutionary mutation and recombination~\citep{fernando2023promptbreeder}, and label-free prompt optimizers reduce dependence on labeled instance-level feedback~\citep{wu2025llm}.

\paragraph{Bayesian optimization over structured and embedded spaces.}
BO is most straightforward in low-dimensional continuous domains, while system prompts are discrete, variable-length, and semantically structured.
Prior work addresses this challenge through the use of lower-dimensional embeddings~\citep{wang2016bayesian,letham2020re}, kernels for structured objects such as strings or graphs~\citep{oh2019combinatorial,moss2020boss,griffiths2023gauche}, and learned latent spaces or deep kernels~\citep{gomez2018automatic,deshwal2021combining,maus2022local,maus2023joint,moss2025return,wilson2016deep}.
These methods motivate our use of a compact representation, but they do not directly provide a full prompt-optimization loop.
Structured kernels still require searching over text, typically by sampling or enumeration, and learned latent spaces usually require auxiliary data or a trained encoder--decoder.
One might consider performing BO over off-the-shelf dense text embeddings. However, this is prohibitive in the small data setting. Fitting a surrogate in thousands of dimensions on only very few observations yields uninformative posteriors. Furthermore, even if dimensionality reduction (e.g., PCA) were applied, decoding an optimized continuous latent vector back into deployable discrete text requires auxiliary trained decoders. In addition, such dimensionality reduction would pick out the most important \emph{general} latent features, but what we are really after are the features most relevant \emph{specifically for system prompt performance} for the target application. \method bypasses both the curse of dimensionality and the inverse-mapping problem by allowing the LLM to construct a low-dimensional, interpretable semantic space where BO targets can be realized natively as deployable prompts via text generation.

\paragraph{Bayesian and surrogate-based prompt optimization.}
Several methods combine BO or surrogate modeling with prompt or instruction search.
InstructZero optimizes soft prompts for an instruction generator~\citep{chen2023instructzero}; BOInG uses BO in relaxed or generator-mediated instruction spaces~\citep{sabbatella2024bayesian}; MIPRO uses a Bayesian surrogate to search over instructions and demonstrations for LM programs~\citep{opsahl2024optimizing}; HbBoPs combines a structural-aware deep-kernel GP with Hyperband for prompt selection~\citep{schneider2024hyperband}; and BOPRO performs BO over fixed embeddings of language-based solutions~\citep{agarwal2025searching}, with related work on prompt and code-generation search~\citep{ballew2025llm,tomar2025exploratory}.
These methods are closely related but target different interfaces, including soft prompts, finite prompt/program configurations, few-shot demonstrations, candidate pools, or fixed embedding spaces.
\method targets deployable hard system prompts under prompt-level scalar feedback, using dynamic elicitation to build the BO representation during optimization.

\section{Method}

\subsection{Problem Setting}
\label{sec:problem-setting}

We consider black-box optimization over system prompts.
Let $f:\mathcal{X}\to\mathbb{R}$ denote an objective that maps a prompt $x\in\mathcal{X}$ to a scalar score $y=f(x)$.
In deployment, this score may be a delayed aggregate metric such as task-completion rate, safety-incident rate, or user satisfaction.
In our experiments, $f(x)$ is benchmark accuracy on a fixed evaluation set, exposed to the optimizer only as a single prompt-level scalar.
The optimizer does not observe per-example labels, individual failures, answer traces, or textual critiques.

The goal is to find a high-scoring prompt within a highly constrained evaluation budget.
Motivated by the typical setting of conducting multiple long-running online experiments in parallel, we use batched optimization with batch size $q$.   
Let $T$ denote the total number of evaluated batches, including the initial seed batch.
The initial dataset $\mathcal{D}_0=\{(x_i,y_i)\}_{i=1}^q$ contains $q$ evaluated seed prompts and corresponds to iteration $t=0$.
For each optimization round $t=1,\ldots,T-1$, the optimizer uses the current history $\mathcal{D}_{t-1}$ to propose a new batch of $q$ prompts, evaluates them, and updates
\[
\mathcal{D}_t
=
\mathcal{D}_{t-1}
\cup
\{(x_{t,j}^{\mathrm{new}},y_{t,j}^{\mathrm{new}})\}_{j=1}^q .
\]
After round $t$, the dataset contains $q(t+1)$ evaluated prompts. The total evaluation budget is $N=qT$.

\subsection{\method}

\begin{algorithm}[ht]
\small
\caption{\method main loop. Expanded subroutines and prompts are in Appendix~\ref{sec:full_algo}.}
\label{alg:reelicit}

\KwIn{$\mathcal{D}_0=\{(x_i,y_i)\}_{i=1}^q$, total evaluated batches $T$, batch size $q$, acquisition function $\alpha$, elicitation rounds $K$, realization budget $M$, tolerance $\tau$}
\KwOut{$x^*=\arg\max_{(x,y)\in\mathcal{D}_{T-1}}y$}

$\mathcal{F}_0 \leftarrow \varnothing$\;

\For{$t=1,\ldots,T-1$}{
  Let $X_{t-1},Y_{t-1}$ be the prompts and scores in $\mathcal{D}_{t-1}$\;

  \For{$k=1,\ldots,K$}{
    $\mathcal{F}_t^{(k)} \leftarrow \DefineFeatures(\mathcal{D}_{t-1},\mathcal{F}_{t-1})$\;
    $Z_t^{(k)} \leftarrow \ExtractFeatures(X_{t-1},\mathcal{F}_t^{(k)})$\;
    $s_t^{(k)} \leftarrow \mathrm{CV}(Z_t^{(k)},Y_{t-1})$\;
  }

  \If{$t>1$}{
    Add incumbent $\mathcal{F}_{t-1}$ as an additional candidate by re-extracting it on $X_{t-1}$ and scoring it by CV\;
  }

  Select $\mathcal{F}_t$ with lowest CV error, and let $Z_t$ be the corresponding embeddings\;

  Fit GP surrogate $\mathcal{M}_t$ on $(Z_t,Y_{t-1})$\;
  $\{z_{t,1}^{\mathrm{new}},\ldots,z_{t,q}^{\mathrm{new}}\}
  \leftarrow
  \arg\max_{z_1,\ldots,z_q\in[0,1]^{d_t}}
  \alpha(z_1,\ldots,z_q\mid\mathcal{M}_t)$\;

  \For{$j=1,\ldots,q$}{
    $x_{t,j}^{\mathrm{new}}
    \leftarrow
    \GenRefine(z_{t,j}^{\mathrm{new}},\mathcal{F}_t,\mathcal{D}_{t-1},M,\tau)$\;
  }

  Evaluate $y_{t,j}^{\mathrm{new}}=f(x_{t,j}^{\mathrm{new}})$ for each $j$\;
  $\mathcal{D}_t
  \leftarrow
  \mathcal{D}_{t-1}
  \cup
  \{(x_{t,j}^{\mathrm{new}},y_{t,j}^{\mathrm{new}})\}_{j=1}^{q}$\;
}

\KwReturn{$x^*=\arg\max_{(x,y)\in\mathcal{D}_{T-1}}y$}
\end{algorithm}

\method performs BO in an LLM-elicited feature space rather than directly over text.
At round $t$, the method constructs an embedding
\[
g_t:\mathcal{X}\to[0,1]^{d_t}
\]
from a set of natural-language feature descriptions $\mathcal{F}_t$.
Each feature describes a task-relevant semantic property of a prompt, and the LLM maps prompts to numerical coordinates along these axes.

Algorithm~\ref{alg:reelicit} summarizes the main optimization loop.
Feature definition and feature extraction use different information.
\DefineFeatures sees the task context and prompt-score history, allowing performance evidence to inform which semantic axes are proposed.
\ExtractFeatures sees prompts and feature definitions, but not scores, so the resulting coordinates are based on prompt content rather than direct outcome leakage.
The selected representation is the candidate feature set with the lowest cross-validation (CV) error for predicting observed scores; when $t>1$, the previous feature set is re-extracted on the enlarged history and included in the batch as an incumbent candidate.

Given the selected embedding, \method fits a Gaussian process surrogate to the embedded observations and uses a batch acquisition function to select target feature vectors.
These vectors are not prompts: they are desired semantic coordinates in the elicited representation.
The LLM realizes each target as a system prompt and refines it using feature-gap feedback, because the LLM-induced inverse map from continuous feature coordinates to text is lossy.
The resulting prompts are evaluated with $f$ and appended to the optimization history.

\section{Theoretical Analysis}
\label{sec:theory}

Our theoretical analysis studies the \emph{reachability} question underlying \method:
when does a point that is near-optimal in an elicited, low-dimensional representation correspond to a good prompt in the original text space?
This question is analogous to reachability analyses in random-embedding BO~\citep{wang2016bayesian}, but differs in that the embedding here is nonlinear, semantic, and constructed by an LLM from the observed prompt-score history.

We formalize this by comparing the elicited embedding to an oracle semantic embedding under which the objective is smooth.
If the elicited embedding preserves the performance-relevant geometry of this oracle representation, then points that are good in the elicited representation are also good for the true objective.
The resulting bound identifies the role of representation error in \method: the better the elicited embedding approximates the oracle embedding, the tighter the connection between optimization in feature space and optimization over prompts.

Let $\mathcal{X}$ be a finite prompt universe.
We assume there exists an injective oracle embedding $g^*: \mathcal{X} \to \mathcal{Z} \subset \mathbb{R}^d$, that, combined with a stationary base latent kernel $k_Z(z, z')$, produces an oracle kernel $k^*(x, x') = k_Z(g^*(x), g^*(x'))$.

\begin{assumption}\label{ass:rkhs}
The true objective function $f: \mathcal{X} \to \mathbb{R}$ resides in the RKHS of the oracle kernel, $f \in \mathcal{H}_{k^*}$, with bounded norm $\|f\|_{k^*} \leq B$.
\end{assumption}

\begin{assumption}\label{ass:smoothness}
The RKHS latent feature map $\phi: \mathcal{Z} \to \mathcal{H}_Z$ is Lipschitz continuous with respect to the latent distance: $\|\phi(z) - \phi(z')\|_{\mathcal{H}_Z} \leq L \|z - z'\|$, $\forall z, z' \in \mathcal{Z}$.
\end{assumption}

We model $f$ through the elicited embedding $g_t$.
Because $g_t$ may differ from the oracle embedding $g^*$, the objective representable through $g_t$ may only approximate the true objective.
For the analysis, let $f_t$ denote the best uniform approximation to $f$ over $\mathcal{X}$ among functions representable through $g_t$ with RKHS norm at most $B$; the full definition is given in Appendix~\ref{app:proof}.
The quality of this approximation depends on the mismatch between $g_t$ and $g^*$.

\begin{assumption}\label{ass:embed}
At iteration $t$, the error between the elicited embedding $g_t(x)$ and the oracle embedding $g^*(x)$ is bounded by $\eta_t \geq 0$ such that $\forall x \in \mathcal{X}$, $\|g^*(x) - g_t(x)\| \leq \eta_t$.
\end{assumption}

This assumption abstracts away the variable-dimensional feature sets used in the implementation and analyzes a fixed latent space in which the elicited and oracle embeddings can be compared.

\begin{definition}
Let $x^*$ be an optimum of $f$: $f(x^*) = \max_{x \in \mathcal{X}} f(x)$.
The $\epsilon$-suboptimal set of $f$ is $S_{\epsilon}(f) = \{x: f(x^*) - f(x) \leq \epsilon\}$.
\end{definition}

We take Assumptions~\ref{ass:rkhs}--\ref{ass:embed} throughout the analysis.
The main result bounds the true optimality gap of any point that is nearly optimal under the approximated objective $f_t$.

\begin{theorem}\label{thm}
Suppose $x_t$ is $\delta$-suboptimal with respect to $f_t$, i.e.,
$\max_{x \in \mathcal{X}} f_t(x) - f_t(x_t) \leq \delta$. Then $x_t$ is $\epsilon$-suboptimal for the true objective, with $\epsilon = \delta + 2BL\eta_t$, i.e., 
$f(x^*) - f(x_{t}) \leq \delta + 2B L \eta_t.$
\end{theorem}

Several insights follow from this result.
The term $\delta$ captures the optimization error within the approximated objective induced by the elicited embedding.
The term $2BL\eta_t$ bounds the representation error: the price paid for optimizing through $g_t$ rather than the oracle embedding $g^*$.
The RKHS norm $B$ and Lipschitz constant $L$ show that this price is smaller when the objective varies smoothly in the latent semantic space.
This aligns with the regime where aggregate prompt metrics are expected to be amenable to BO: small semantic changes in a system prompt can produce gradual changes in average downstream performance.

Re-elicitation can be viewed as a mechanism for reducing $\eta_t$ as data accumulate.
Each new batch gives the LLM additional evidence about which semantic properties distinguish high- and low-performing prompts.
When this evidence leads to an embedding closer to the oracle performance-relevant representation, the bound tightens.
Thus, the theorem identifies a concrete pathway by which dynamic elicitation can improve BO over prompts: by reducing representation error, it tightens the guarantee connecting feature-space optima to true prompt quality.

\section{Experiments}

\subsection{Setup}
We evaluate \method on ten system prompt optimization tasks under the aggregate feedback interface of Section~\ref{sec:problem-setting}. 
We study this aggregate feedback interface in a controlled offline setting.
Our experiments instantiate the objective $f(x)$ using benchmark accuracy on fixed evaluation sets, but the optimizer observes only a single scalar score for each prompt. It does not observe per-example labels, individual failures, answer traces, or textual critiques. This protocol isolates the methodological question we target: how can one perform sample-efficient search over system prompts when only prompt-level aggregate feedback is available?

\paragraph{Baselines.}
We compare against aggregate-only adaptations of four representative hard-prompt search methods: APE-style sampling~\citep{zhou2022large}, which samples history-free prompts; OPRO~\citep{yang2023large}, which conditions generation on score-sorted prompt histories; PromptBreeder~\citep{fernando2023promptbreeder}, which mutates and recombines prompts from a fitness-sorted population; and TextGrad-style refinement~\citep{yuksekgonul2024textgrad}, which critiques the scalar prompt-score trajectory and proposes variants.
These baselines test whether embedding by elicitation and BO-guided realization improve over common LLM-guided hard-prompt search when all methods receive the same prompt-level scalar feedback.
Closely related BO-based prompt optimizers are discussed in Section~\ref{sec:related_work}; many target soft prompts, finite instruction--demonstration configurations, candidate pools, fixed embeddings, or multi-fidelity validation subsets rather than deployable hard system prompts.
All baselines share the task context, initial evaluated dataset $\mathcal{D}_0$, scalar-score history, and target-model evaluation budget with \method.
Full aggregate-only baseline adaptations and prompts are given in Appendix~\ref{sec:baselines_details}.

\paragraph{Benchmarks and protocol.}
The tasks are drawn from GSM8K~\citep{cobbe2021training}, MMLU~\citep{hendryckstest2021}, and BIG-Bench Hard (BBH)~\citep{suzgun2022challenging}.
GSM8K and MMLU use fixed 500-question subsamples; the remaining eight tasks are BBH tasks with 250 examples each.
Appendix~\ref{sec:benchmarks_details} lists the task descriptions and optimizer-facing task contexts.
Each run has budget $N=30$ target-model evaluations: $q=5$ prompts per batch across $T=6$ evaluated batches, where the first batch is the shared seed dataset $\mathcal{D}_0$ generated by Algorithm~\ref{alg:d0} and the remaining five batches are optimization rounds.
Unless otherwise noted, we report means and 95\% confidence intervals over 30 independent seeds.

\paragraph{LLM use and evaluation.}
LLMs are core experimental components.
The optimizer LLM is Llama 3.3 70B Instruct, and the target LLM is Llama 3.1 8B Instruct.
In \method, the optimizer LLM elicits feature axes, extracts prompt coordinates, realizes BO-selected feature targets as prompts, and refines feature gaps; in the baselines, the same optimizer LLM generates candidates or critiques according to the corresponding baseline.
The target LLM is queried only through the evaluation function $f$: each evaluation runs the target model with system prompt $x$ on the fixed benchmark subset using zero-shot greedy decoding and returns aggregate accuracy.
We use the smaller target model to avoid task saturation and keep evaluation tractable in terms of wall time and token usage.

Our primary budget is the number of target-model evaluations, since each evaluation runs the target model over hundreds of benchmark examples.
\method uses additional optimizer-side LLM calls for elicitation, extraction, realization, and refinement, so the experiments evaluate sample efficiency in target evaluations rather than total LLM-call efficiency. In practice, the cost of optimization is typically dwarfed by the cost of evaluating $f$, which often requires large scale online A/B testing across large user populations.  
Shared history subsampling, optimizer-side temperatures, hyperparameters, information-access controls, and exact prompts are reported in Appendices~\ref{app:eval_protocol} and \ref{app:prompts}.

\subsection{Main Result}
Table~\ref{tab:main-performance} reports the final best score achieved by each method under the shared 30-evaluation budget.
For each seed and method, we take the best prompt found by the end of optimization and report the mean and 95\% confidence interval across seeds.
\method has the strongest overall performance profile: it is either the numerically best method or statistically indistinguishable from the best method on all ten tasks.
The absolute margins vary across tasks, which is expected in this small-budget setting, but the aggregate comparison in Table~\ref{tab:win-or-tie} shows that the effect is consistent across runs.
Across task-seed pairs, \method on average matches or exceeds APE, OPRO, PromptBreeder, and TextGrad as baselines.
This suggests that the elicited representation and BO-guided realization loop improve robustness across tasks rather than producing gains on only a single benchmark.

\begin{table}[ht]
\centering
\small
\begin{tabular}{llllll}
\toprule
 & GSM8K & MMLU & Boolean Expr. & Disambig. QA & Tracking \\
\midrule
\textbf{\method}& \textbf{0.833 $\pm$ 0.006} & \textbf{0.571 $\pm$ 0.031} & \textbf{0.768 $\pm$ 0.009} & \textbf{0.551 $\pm$ 0.006} & \textbf{0.204 $\pm$ 0.003} \\
APE & \textbf{0.830 $\pm$ 0.005} & \textbf{0.595 $\pm$ 0.016} & 0.719 $\pm$ 0.013 & 0.514 $\pm$ 0.008 & 0.196 $\pm$ 0.005 \\
OPRO & 0.818 $\pm$ 0.008 & 0.568 $\pm$ 0.014 & 0.650 $\pm$ 0.014 & 0.524 $\pm$ 0.008 & \textbf{0.205 $\pm$ 0.004} \\
PromptBreeder & \textbf{0.833 $\pm$ 0.005} & 0.508 $\pm$ 0.044 & 0.643 $\pm$ 0.017 & 0.516 $\pm$ 0.009 & 0.163 $\pm$ 0.015 \\
TextGrad & \textbf{0.827 $\pm$ 0.008} & 0.516 $\pm$ 0.044 & 0.669 $\pm$ 0.019 & 0.532 $\pm$ 0.009 & 0.186 $\pm$ 0.009 \\
\bottomrule
\end{tabular}
\vspace{1ex}

\begin{tabular}{llllll}
\toprule
 & Penguins & Causal Judg. & Formal Fall. & Snarks & Hyperbaton \\
\midrule
\textbf{\method} & \textbf{0.441 $\pm$ 0.010} & \textbf{0.584 $\pm$ 0.005} & \textbf{0.581 $\pm$ 0.007} & \textbf{0.551 $\pm$ 0.009} & \textbf{0.796 $\pm$ 0.006} \\
APE & \textbf{0.434 $\pm$ 0.001} & 0.570 $\pm$ 0.002 & 0.576 $\pm$ 0.002 & 0.527 $\pm$ 0.009 & 0.770 $\pm$ 0.008 \\
OPRO & \textbf{0.439 $\pm$ 0.004} & 0.571 $\pm$ 0.004 & \textbf{0.582 $\pm$ 0.002} & 0.537 $\pm$ 0.006 & 0.783 $\pm$ 0.007 \\
PromptBreeder & 0.293 $\pm$ 0.029 & 0.558 $\pm$ 0.006 & 0.546 $\pm$ 0.005 & 0.539 $\pm$ 0.007 & \textbf{0.802 $\pm$ 0.008} \\
TextGrad & 0.331 $\pm$ 0.024 & 0.560 $\pm$ 0.007 & 0.539 $\pm$ 0.006 & \textbf{0.554 $\pm$ 0.008} & 0.791 $\pm$ 0.007 \\
\bottomrule
\end{tabular}
\vspace{0.05in}
\caption{Main performance comparison.
Entries are final best scores after 30 prompt evaluations, reported as mean $\pm$ 95\% confidence interval over seeds.
Bold cells indicate the numerical best method and methods not significantly different from the best.
\method is statistically better or tied for best on all tasks and has the strongest aggregate pairwise win-or-tie profile in Table~\ref{tab:win-or-tie}.}
\label{tab:main-performance}
\vspace{-0.2in}
\end{table}

\begin{table}[ht]
\centering
\small
\begin{tabular}{l|ccccc|c}
\toprule
 & ReElicit & APE & OPRO & PromptBreeder & TextGrad & Avg \\
\midrule
\textbf{ReElicit} & \textbf{--} & \textbf{0.78} & \textbf{0.79} & \textbf{0.85} & \textbf{0.82} & \textbf{0.81} \\
APE & 0.30 & -- & 0.60 & 0.71 & 0.63 & 0.56 \\
OPRO & 0.31 & 0.66 & -- & 0.71 & 0.63 & 0.58 \\
PromptBreeder & 0.22 & 0.44 & 0.45 & -- & 0.48 & 0.40 \\
TextGrad & 0.26 & 0.48 & 0.51 & 0.70 & -- & 0.49 \\
\bottomrule
\end{tabular}
\vspace{0.05in}
\caption{Pairwise win-or-tie rate over final optimization results.
Each cell is the fraction of task-seed pairs where the row method matches or exceeds the column method in final best score.
Ties count for both methods, so opposite-direction entries need not sum to one.
The final column averages each row's off-diagonal entries.
\method has the highest average win-or-tie rate, indicating the most consistent aggregate performance across baselines.}
\label{tab:win-or-tie}
\vspace{-0.3in}
\end{table}

\subsection{Component Analysis}
\label{sec:component_analysis}

We next ask whether the elicited representation has the properties needed for BO over prompts: repeatable feature extraction, conservative adaptation across rounds, low dimensionality, predictive surrogate fit, and actionable targets for prompt generation.

\paragraph{Feature generation and extraction.}
A useful dynamic representation should be consistent under repeated extraction, but still able to adapt as new prompt-score evidence arrives.
We measure consistency by re-extracting feature values three times under the same selected feature definitions, and measure adaptation using cross-iteration linear centered kernel alignment (CKA)~\citep{kornblith2019similarity} on prompts shared across adjacent histories.

\begin{figure}[ht]
\centering
  \begin{subfigure}[ht]{0.45\linewidth}
    \centering
    \includegraphics[width=\linewidth]{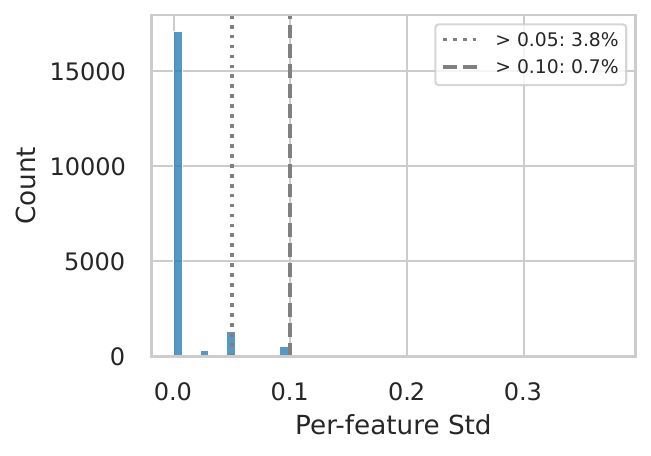}
    \caption{Per-feature std}
    \label{fig:per-feature-std}
  \end{subfigure}
  \hfill
  \begin{subfigure}[ht]{0.45\linewidth}
    \centering
    \includegraphics[width=\linewidth]{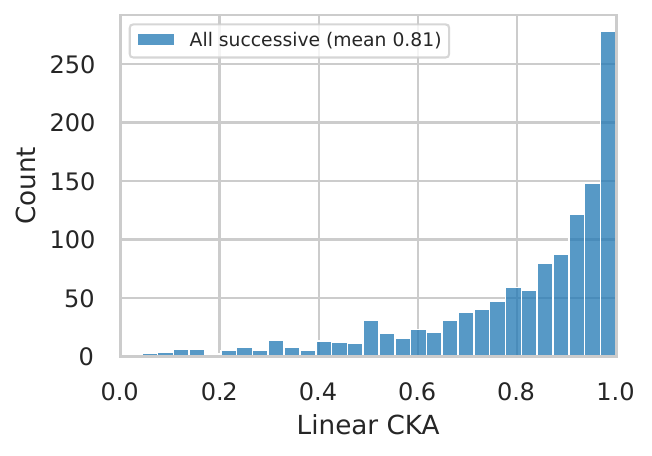}
    \caption{Cross-iteration linear CKA}
    \label{fig:cka-stability}
  \end{subfigure}
\caption{Feature stability and conservative adaptation.
(a) Repeated extraction is highly stable: fewer than 4\% of feature values have standard deviation above 0.05, and 0.7\% exceed 0.1.
(b) Cross-iteration CKA has mean 0.81, indicating high but imperfect alignment between successive selected representations.}
\label{fig:feature-stability}
\vspace{-0.1in}
\end{figure}

Figure~\ref{fig:feature-stability} shows that extraction noise is small and that re-elicited representations adapt conservatively.
The CKA distribution is highly aligned but not concentrated at one, suggesting that the representation preserves enough geometry for BO to accumulate evidence while still changing as additional prompt-score pairs reveal new semantic distinctions.
Appendix Table~\ref{tab:case-study} gives a qualitative example of the tasks evaluated: selected features evolve across iterations while cross-validation MSE decreases.

\paragraph{Predictive features and actionable targets.}
The BO loop needs the elicited feature space to be useful in two senses.
First, the representation must make prompt scores predictable from very small histories: during optimization, the GP surrogate is fit with only 5 to 25 observed prompt-score pairs.
Second, BO-selected feature targets must be realizable as natural-language prompts in a way that matters for downstream improvement.

\begin{figure}[ht]
\centering
  \begin{subfigure}[ht]{0.42\linewidth}
    \centering
    \includegraphics[width=\linewidth]{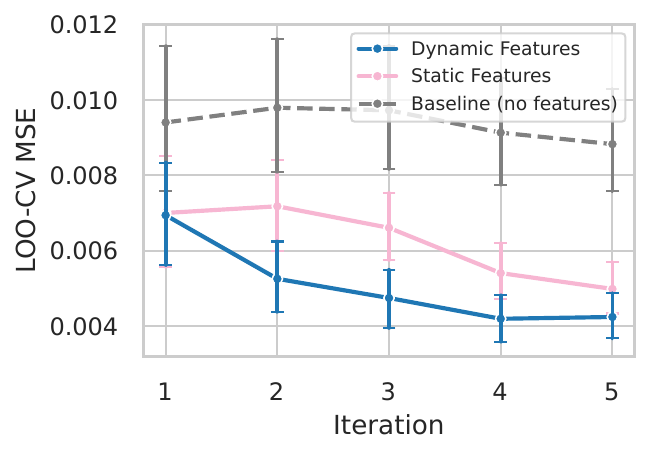}
    \caption{GP surrogate fit}
    \label{fig:gp-fit}
  \end{subfigure}
  \hfill
  \begin{subfigure}[ht]{0.43\linewidth}
    \centering
    \includegraphics[width=\linewidth]{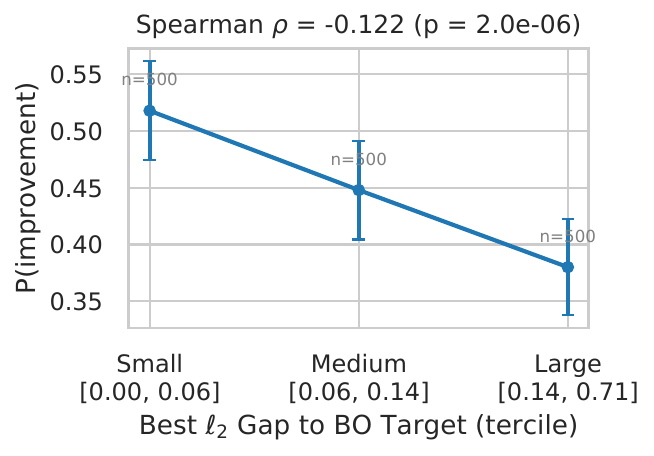}
    \caption{Gap vs.\ improvement probability}
    \label{fig:gap-vs-improve}
  \end{subfigure}
\caption{Predictiveness and actionability of the elicited feature space.
(a) Dynamic features yield lower GP cross-validation MSE than static initial features. Both are significantly better than a mean-predictor baseline.
(b) Smaller final $\ell_2$ gaps to BO targets are associated with higher improvement probability.}
\label{fig:predictive-actionable}
\vspace{-0.2in}
\end{figure}

Figure~\ref{fig:predictive-actionable} shows that the elicited representation is useful both for modeling and for generation.
Dynamic elicited features reduce GP cross-validation MSE relative to static initial features and a mean-predictor baseline, supporting the value of re-eliciting features as new prompt-score evidence accumulates.
Separately, prompts whose final extracted features are closer to the BO-selected target are more likely to improve the best-so-far score in the next evaluation. 
This supports feature-gap refinement as a useful mechanism for translating BO targets back into deployable prompts.
Appendix~\ref{app:additional-results} reports additional diagnostics and analyses.

\subsection{Ablation Study}
\label{sec:ablation}

We ablate four design choices.
\textbf{No Refinement} accepts the initial realization of each BO target without feature-gap refinement.
\textbf{No BO} replaces acquisition optimization with uniform sampling in the elicited feature space.
\textbf{Static Features} freezes the first selected feature set and removes dynamic re-elicitation.
\textbf{Independent Extraction} rates one prompt at a time instead of extracting features jointly in batches.

Figure~\ref{fig:ablation-convergence} summarizes convergence across the ten benchmarks, and Appendix Table~\ref{tab:ablation-paired-delta} reports paired final-score differences.
The clearest performance contributions come from feature-gap refinement and BO target selection: removing refinement or replacing BO with random feature-space sampling produces the largest drops in some of the best-so-far optimization curves as well as in final scores.
These results support the two main algorithmic roles of the loop: realizing BO-selected semantic targets as prompts, and using uncertainty-aware acquisition rather than unguided feature-space sampling.
Freezing the feature set also decreases optimization performance, but the effect is milder than removing refinement or replacing BO with random feature-space sampling. This suggests that even the initial dataset often lets the LLM elicit semantically meaningful features that support effective optimization. Although static features are less predictive than their dynamically re-elicited counterparts, as shown in Figure~\ref{fig:gp-fit}, GPs fitted on static features still achieve low cross-validation MSE relative to the baseline.

Independent Extraction isolates an implementation choice rather than a core algorithmic component.
Its similar final performance suggests that joint extraction preserves optimization quality while reducing feature-extraction cost.
With extraction batch size $b=10$, extracting features for $n$ prompts requires $\lceil n/b\rceil$ LLM calls per feature set rather than $n$ calls; at the largest history size in our experiments, this is 3 calls rather than 25 per feature set.

\begin{figure}[ht]
\centering
\includegraphics[width=\textwidth]{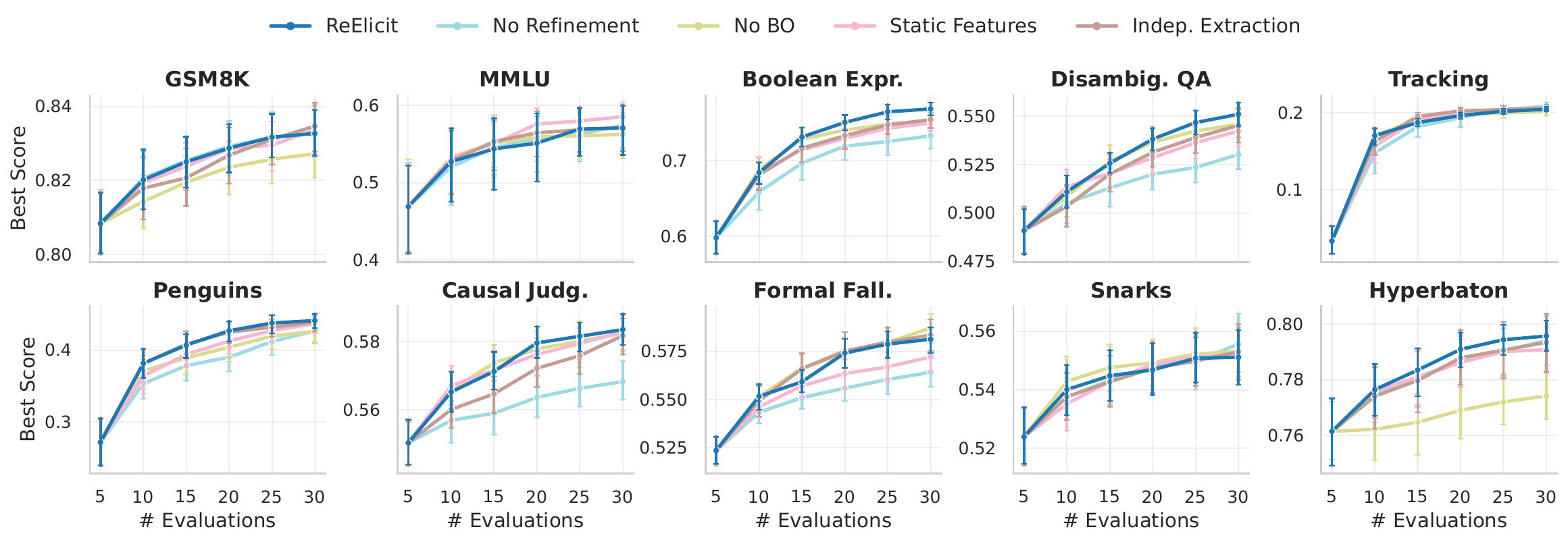}
\caption{Ablation convergence curves.
Removing feature-gap refinement or replacing BO with random feature-space sampling produces the clearest degradation.
Static features and independent extraction yield smaller differences, which are quantified in Appendix Table~\ref{tab:ablation-paired-delta}.}
\label{fig:ablation-convergence}
\vspace{-0.2in}
\end{figure}

\section{Discussion}

\method treats system-prompt tuning as black-box optimization with a learned semantic coordinate system. Prompts are discrete text objects, but the search directions that matter are semantic, task-dependent, and unavailable in advance. \method uses the optimizer LLM to elicit compact feature axes from prompt-score history, BO to select target coordinates, and the LLM again to realize and feature-gap-refine those targets into deployable prompts. This couples LLM prior knowledge and text generation with BO's uncertainty-aware search under scarce aggregate evaluations.

Across ten tasks, \method has the strongest aggregate profile: the highest pairwise win-or-tie rate and performance that is numerically best or not significantly worse than the numerical best on every task. Diagnostics support the representation view: elicited spaces are stable under repeated extraction but adapt across rounds, improve GP surrogate fit, and yield BO targets whose realized feature gaps predict improvement. Ablations show the largest contributions from BO target selection and feature-gap refinement, while dynamic re-elicitation improves the surrogate's predictive structure.

The main limitations concern evaluation scope.
For reproducibility and controlled comparison, we instantiate aggregate feedback using offline benchmark accuracy on GSM8K, MMLU, and BBH.
This is standard in prompt-optimization work and lets us compare methods under identical scalar-feedback budgets, but it is only a proxy for motivating deployment objectives such as user satisfaction, retention, or safety-incident rate.
Testing those objectives would require live systems, changing user populations, and less standardized baselines.

Related APO/BO methods discussed in Section~\ref{sec:related_work} often optimize different objects or assume different interfaces, such as soft prompts, instruction-demonstration program spaces, candidate pools, fixed embeddings, validation subsets, or instance-level feedback.
Adapting them would change what information the optimizer receives or what search space is being optimized, so our empirical claim is an apples-to-apples comparison within the aggregate-only hard-prompt setting rather than a universal prompt-optimization leaderboard. 
\method also uses additional optimizer-side LLM calls. Therefore, our efficiency claim is target-evaluation efficiency, most relevant when each evaluation is a costly aggregate measurement.

More broadly, embedding by elicitation suggests a recipe for optimizing language-describable structured artifacts.
Whenever candidates can be specified in text, constraints can be expressed textually, and evaluation is an expensive scalar measurement, an LLM can construct an interpretable adaptive search space while BO performs sample-efficient exploration within it.
In principle, this pattern could apply beyond system prompts to other modalities (such as images, audio, or video by using multi-modal models) and other applications such as agentic tool-use instructions or evaluation rubrics.
\vspace{-0.1in}

\paragraph{Broader impacts.}
System prompt optimization can improve reliability and controllability by making prompt tuning more systematic under limited feedback.
The main risks are objective-driven: like other black-box optimizers, the method could optimize harmful goals or over-optimize a poorly specified metric in ways that weaken unmeasured safety, fairness, privacy, or robustness constraints. 
Deployment-facing use should therefore treat metric design as part of the safety problem, with human review, held-out evaluation, safety-relevant slices, robustness checks, and explicit tests that gains on the target metric do not come from violating constraints omitted from the aggregate objective.
Our experiments are offline benchmark evaluations and do not deploy optimized prompts to users.

\bibliography{references}

@inproceedings{balandat2020botorch,
  title={{BoTorch: A Framework for Efficient Monte-Carlo Bayesian Optimization}},
  author={Balandat, Maximilian and Karrer, Brian and Jiang, Daniel R. and Daulton, Samuel and Letham, Benjamin and Wilson, Andrew Gordon and Bakshy, Eytan},
  booktitle = {Advances in Neural Information Processing Systems 33},
  year={2020},
  url = {http://arxiv.org/abs/1910.06403}
}

@article{wang2016bayesian,
  title={Bayesian optimization in a billion dimensions via random embeddings},
  author={Wang, Ziyu and Hutter, Frank and Zoghi, Masrour and Matheson, David and De Feitas, Nando},
  journal={Journal of Artificial Intelligence Research},
  volume={55},
  pages={361--387},
  year={2016}
}

@inproceedings{ramnath2025systematic,
  title={A systematic survey of automatic prompt optimization techniques},
  author={Ramnath, Kiran and Zhou, Kang and Guan, Sheng and Mishra, Soumya Smruti and Qi, Xuan and Shen, Zhengyuan and Wang, Shuai and Woo, Sangmin and Jeoung, Sullam and Wang, Yawei and others},
  booktitle={Proceedings of the 2025 Conference on Empirical Methods in Natural Language Processing},
  pages={33066--33098},
  year={2025}
}

@article{letham2020re,
  title={Re-examining linear embeddings for high-dimensional Bayesian optimization},
  author={Letham, Ben and Calandra, Roberto and Rai, Akshara and Bakshy, Eytan},
  journal={Advances in neural information processing systems},
  volume={33},
  pages={1546--1558},
  year={2020}
}

@inproceedings{kornblith2019similarity,
  title={Similarity of neural network representations revisited},
  author={Kornblith, Simon and Norouzi, Mohammad and Lee, Honglak and Hinton, Geoffrey},
  booktitle={International conference on machine learning},
  pages={3519--3529},
  year={2019},
  organization={PMLR}
}

@inproceedings{wilson2016deep,
  title={Deep kernel learning},
  author={Wilson, Andrew Gordon and Hu, Zhiting and Salakhutdinov, Ruslan and Xing, Eric P},
  booktitle={Artificial intelligence and statistics},
  pages={370--378},
  year={2016},
  organization={PMLR}
}

@article{maus2022local,
  title={Local latent space bayesian optimization over structured inputs},
  author={Maus, Natalie and Jones, Haydn and Moore, Juston and Kusner, Matt J and Bradshaw, John and Gardner, Jacob},
  journal={Advances in neural information processing systems},
  volume={35},
  pages={34505--34518},
  year={2022}
}

@article{maus2023joint,
  title={Joint composite latent space bayesian optimization},
  author={Maus, Natalie and Lin, Zhiyuan Jerry and Balandat, Maximilian and Bakshy, Eytan},
  journal={arXiv preprint arXiv:2311.02213},
  year={2023}
}

@inproceedings{zhou2022large,
  title={Large language models are human-level prompt engineers},
  author={Zhou, Yongchao and Muresanu, Andrei Ioan and Han, Ziwen and Paster, Keiran and Pitis, Silviu and Chan, Harris and Ba, Jimmy},
  booktitle={The eleventh international conference on learning representations},
  year={2022}
}

@inproceedings{yang2023large,
  title={Large language models as optimizers},
  author={Yang, Chengrun and Wang, Xuezhi and Lu, Yifeng and Liu, Hanxiao and Le, Quoc V and Zhou, Denny and Chen, Xinyun},
  booktitle={The Twelfth International Conference on Learning Representations},
  year={2023}
}

@article{yuksekgonul2024textgrad,
  title={Textgrad: Automatic" differentiation" via text},
  author={Yuksekgonul, Mert and Bianchi, Federico and Boen, Joseph and Liu, Sheng and Huang, Zhi and Guestrin, Carlos and Zou, James},
  journal={arXiv preprint arXiv:2406.07496},
  year={2024}
}

@inproceedings{pryzant2023automatic,
  title={Automatic prompt optimization with “gradient descent” and beam search},
  author={Pryzant, Reid and Iter, Dan and Li, Jerry and Lee, Yin and Zhu, Chenguang and Zeng, Michael},
  booktitle={Proceedings of the 2023 conference on empirical methods in natural language processing},
  pages={7957--7968},
  year={2023}
}

@article{fernando2023promptbreeder,
  title={Promptbreeder: Self-referential self-improvement via prompt evolution},
  author={Fernando, Chrisantha and Banarse, Dylan and Michalewski, Henryk and Osindero, Simon and Rockt{\"a}schel, Tim},
  journal={arXiv preprint arXiv:2309.16797},
  year={2023}
}

@inproceedings{opsahl2024optimizing,
  title={Optimizing instructions and demonstrations for multi-stage language model programs},
  author={Opsahl-Ong, Krista and Ryan, Michael J and Purtell, Josh and Broman, David and Potts, Christopher and Zaharia, Matei and Khattab, Omar},
  booktitle={Proceedings of the 2024 Conference on Empirical Methods in Natural Language Processing},
  pages={9340--9366},
  year={2024}
}

@article{wu2025llm,
  title={Llm prompt duel optimizer: Efficient label-free prompt optimization},
  author={Wu, Yuanchen and Verma, Saurabh and Lee, Justin and Xiong, Fangzhou and Zhang, Poppy and Awadelkarim, Amel and Chen, Xu and Yuan, Yubai and Hill, Shawndra},
  journal={arXiv preprint arXiv:2510.13907},
  year={2025}
}

@article{moss2025return,
  title={Return of the latent space COWBOYS: Re-thinking the use of VAEs for Bayesian optimisation of structured spaces},
  author={Moss, Henry B and Ober, Sebastian W and Diethe, Tom},
  journal={arXiv preprint arXiv:2507.03910},
  year={2025}
}

@article{chen2023instructzero,
    title={Instructzero: Efficient instruction optimization for black-box large language models}, 
    author={Chen, Lichang and Chen, Jiuhai and Goldstein, Tom and Huang, Heng and Zhou, Tianyi}, 
    journal={arXiv preprint arXiv:2306.03082}, 
    year={2023} 
}

@article{sabbatella2024bayesian,
  title={Bayesian optimization for instruction generation},
  author={Sabbatella, Antonio and Archetti, Francesco and Ponti, Andrea and Giordani, Ilaria and Candelieri, Antonio},
  journal={Applied Sciences},
  volume={14},
  number={24},
  pages={11865},
  year={2024},
  publisher={MDPI}
}

@article{ballew2025llm,
  title={LLM Based Bayesian Optimization for Prompt Search},
  author={Ballew, Adam and Wang, Jingbo and Ren, Shaogang},
  journal={arXiv preprint arXiv:2510.04384},
  year={2025}
}

@article{tomar2025exploratory,
  title={An Exploratory Study of Bayesian Prompt Optimization for Test-Driven Code Generation with Large Language Models},
  author={Tomar, Shlok and Deshwal, Aryan and Villalovoz, Ethan and Fazzini, Mattia and Cai, Haipeng and Doppa, Janardhan Rao},
  journal={arXiv preprint arXiv:2512.15076},
  year={2025}
}

@article{agrawal2025gepa,
  title={Gepa: Reflective prompt evolution can outperform reinforcement learning},
  author={Agrawal, Lakshya A and Tan, Shangyin and Soylu, Dilara and Ziems, Noah and Khare, Rishi and Opsahl-Ong, Krista and Singhvi, Arnav and Shandilya, Herumb and Ryan, Michael J and Jiang, Meng and others},
  journal={arXiv preprint arXiv:2507.19457},
  year={2025}
}

@inproceedings{agarwal2025searching,
  title={Searching for optimal solutions with LLMs via bayesian optimization},
  author={Agarwal, Dhruv and Arivazhagan, Manoj Ghuhan and Das, Rajarshi and Swamy, Sandesh and Khosla, Sopan and Gangadharaiah, Rashmi},
  booktitle={The Thirteenth International Conference on Learning Representations},
  year={2025}
}

@article{yen2025data,
  title={Data mixture optimization: A multi-fidelity multi-scale bayesian framework},
  author={Yen, Thomson and Siah, Andrew Wei Tung and Chen, Haozhe and Peng, Tianyi and Guetta, Daniel and Namkoong, Hongseok},
  journal={arXiv preprint arXiv:2503.21023},
  year={2025}
}

@inproceedings{olson2025ax,
  title={Ax: A platform for adaptive experimentation},
  author={Olson, Miles and Santorella, Elizabeth and Tiao, Louis C and Cakmak, Sait and Garrard, Mia and Daulton, Samuel and Lin, Zhiyuan Jerry and Ament, Sebastian and Beckerman, Bernard and Onofrey, Eric and others},
  booktitle={AutoML 2025 ABCD Track},
  year={2025}
}

@article{letham2019constrained,
  title={Constrained Bayesian optimization with noisy experiments},
  author={Letham, Benjamin and Karrer, Brian and Ottoni, Guilherme and Bakshy, Eytan},
  journal={Bayesian Analysis},
  year={2019}
}

@article{shahriari2015taking,
  title={Taking the human out of the loop: A review of Bayesian optimization},
  author={Shahriari, Bobak and Swersky, Kevin and Wang, Ziyu and Adams, Ryan P and De Freitas, Nando},
  journal={Proceedings of the IEEE},
  volume={104},
  number={1},
  pages={148--175},
  year={2015},
  publisher={IEEE}
}

@article{oh2019combinatorial,
  title={Combinatorial bayesian optimization using the graph cartesian product},
  author={Oh, Changyong and Tomczak, Jakub and Gavves, Efstratios and Welling, Max},
  journal={Advances in Neural Information Processing Systems},
  volume={32},
  year={2019}
}

@article{moss2020boss,
  title={Boss: Bayesian optimization over string spaces},
  author={Moss, Henry and Leslie, David and Beck, Daniel and Gonzalez, Javier and Rayson, Paul},
  journal={Advances in neural information processing systems},
  volume={33},
  pages={15476--15486},
  year={2020}
}

@article{griffiths2023gauche,
  title={GAUCHE: a library for Gaussian processes in chemistry},
  author={Griffiths, Ryan-Rhys and Klarner, Leo and Moss, Henry and Ravuri, Aditya and Truong, Sang and Du, Yuanqi and Stanton, Samuel and Tom, Gary and Rankovic, Bojana and Jamasb, Arian and others},
  journal={Advances in Neural Information Processing Systems},
  volume={36},
  pages={76923--76946},
  year={2023}
}

@article{deshwal2021combining,
  title={Combining latent space and structured kernels for Bayesian optimization over combinatorial spaces},
  author={Deshwal, Aryan and Doppa, Jana},
  journal={Advances in neural information processing systems},
  volume={34},
  pages={8185--8200},
  year={2021}
}

@article{gomez2018automatic,
  title={Automatic chemical design using a data-driven continuous representation of molecules},
  author={G{\'o}mez-Bombarelli, Rafael and Wei, Jennifer N and Duvenaud, David and Hern{\'a}ndez-Lobato, Jos{\'e} Miguel and S{\'a}nchez-Lengeling, Benjam{\'\i}n and Sheberla, Dennis and Aguilera-Iparraguirre, Jorge and Hirzel, Timothy D and Adams, Ryan P and Aspuru-Guzik, Al{\'a}n},
  journal={ACS central science},
  volume={4},
  number={2},
  pages={268--276},
  year={2018},
  publisher={ACS Publications}
}

@article{schneider2024hyperband,
  title={Hyperband-based Bayesian optimization for black-box prompt selection},
  author={Schneider, Lennart and Wistuba, Martin and Klein, Aaron and Golebiowski, Jacek and Zappella, Giovanni and Merra, Felice Antonio},
  journal={arXiv preprint arXiv:2412.07820},
  year={2024}
}

@misc{cobbe2021training,
      title={Training Verifiers to Solve Math Word Problems},
      author={Karl Cobbe and Vineet Kosaraju and Mohammad Bavarian and Jacob Hilton and Reiichiro Nakano and Christopher Hesse and John Schulman},
      year={2021},
      eprint={2110.14168},
      archivePrefix={arXiv},
      primaryClass={cs.LG}
}

@article{hendryckstest2021,
  title={Measuring Massive Multitask Language Understanding},
  author={Dan Hendrycks and Collin Burns and Steven Basart and Andy Zou and Mantas Mazeika and Dawn Song and Jacob Steinhardt},
  journal={Proceedings of the International Conference on Learning Representations (ICLR)},
  year={2021}
}

@article{suzgun2022challenging,
  title={Challenging BIG-Bench Tasks and Whether Chain-of-Thought Can Solve Them},
  author={Suzgun, Mirac and Scales, Nathan and Sch{\"a}rli, Nathanael and Gehrmann, Sebastian and Tay, Yi and Chung, Hyung Won and Chowdhery, Aakanksha and Le, Quoc V and Chi, Ed H and Zhou, Denny and Wei, Jason},
  journal={arXiv preprint arXiv:2210.09261},
  year={2022}
}
\bibliographystyle{plainnat}

\clearpage

\appendix
\section*{Appendix}

\section{Proof of theoretical result}
\label{app:proof}
The theoretical analysis begins from Assumption \ref{ass:rkhs}, that the true objective function $f$ resides in the RKHS of the oracle kernel, $\mathcal{H}_{k^*}$. Because the oracle kernel is a pullback kernel, its RKHS $\mathcal{H}_{k^*}$ is the space of functions induced by the composition $f = h \circ g^*$, where $h \in \mathcal{H}_Z$. For any $f \in \mathcal{H}_{k^*}$, the norm is defined as $\|f\|_{k^*} = \inf \{ \|h\|_{\mathcal{H}_Z} : f = h \circ g^* \}$. For $\|f\|_{k^*} \leq B$, there exists a weight vector $w^* \in \mathcal{H}_Z$ with $\|w^*\|_{\mathcal{H}_Z} = \|f\|_{k^*} \leq B$ such that $f(x) = \langle w^*, \phi(g^*(x)) \rangle_{\mathcal{H}_Z}$, where $\phi$ is the feature map associated with $k_Z$.
The latent space $\mathcal{Z}$ is compact and contains the images of $g^*$ and all elicited embeddings under consideration.

The oracle embedding $g^*$ is, of course, unknown, and we do modeling and optimizing within the elicited embedding $g_t$. The embedding $g_t$ in general will not be able to fully represent $f$. We define the approximation error $\mathcal{E}_t$ as the infimum of the uniform distance between $f$ and any function representable in the elicited embedding with RKHS norm at most $B$:
\begin{equation*}
    \mathcal{E}_t = \inf_{w \in \mathcal{H}_Z, \|w\|_{\mathcal{H}_Z} \leq B} \sup_{x \in \mathcal{X}} |f(x) - \langle w, \phi(g_t(x)) \rangle|.
\end{equation*}
The norm constraint matches the complexity of the oracle representation, ensuring the bound reflects approximation error rather than overfitting. We define the approximated function $f_t$ as a minimizer of this quantity, and thus a function whose error equals $\mathcal{E}_t$. We now bound the representation error between the true objective function and the approximated function on the embedding.

\begin{lemma}\label{lem:f_approximation}
For all $x \in \mathcal{X}$, the pointwise difference between the true objective function and the approximated function on the embedding is bounded as
\begin{equation*}
    |f(x) - f_t(x)| \leq B L \eta_t.
\end{equation*}
\end{lemma}

\begin{proof}
Consider the following structural surrogate for $f$ that applies the oracle weights $w^*$ to the elicited embedding:
\begin{equation*}
    \bar{f}_t = \langle w^*, \phi(g_t(x)) \rangle_{\mathcal{H}_Z}.
\end{equation*}
Let $w_t' \in \mathcal{H}_Z$ be the weight vector corresponding to $f_t$. By the definition of $f_t$ as the best available approximation, we have that
\begin{align*}
    \mathcal{E}_t &= \sup_{x \in \mathcal{X}} |f(x) - \langle w_t', \phi(g_t(x)) \rangle_{\mathcal{H}_Z} |\\
    &\leq \sup_{x \in \mathcal{X}} |f(x) - \langle w^*, \phi(g_t(x)) \rangle_{\mathcal{H}_Z} |
\end{align*}
Thus, by the linearity of the inner product and the Cauchy-Schwarz inequality,
\begin{align*}
    \sup_{x \in \mathcal{X}} |f(x) - f_t(x)| &\leq \sup_{x \in \mathcal{X}} |f(x) - \langle w^*, \phi(g_t(x)) \rangle_{\mathcal{H}_Z} |\\
    &= \sup_{x \in \mathcal{X}} \Big| \langle w^*, \phi(g^*(x)) \rangle_{\mathcal{H}_Z} - \langle w^*, \phi(g_t(x)) \rangle_{\mathcal{H}_Z} \Big| \\
    &\leq \sup_{x \in \mathcal{X}} \|w^*\|_{\mathcal{H}_Z} \|\phi(g^*(x)) - \phi(g_t(x))\|_{\mathcal{H}_Z}.
\end{align*}

By Assumption \ref{ass:rkhs}, we have that $\|w^*\|_{\mathcal{H}_Z} \leq B$. Applying Assumption \ref{ass:smoothness} bounds the feature distance by the latent space distance, which is in turn bounded according to Assumption \ref{ass:embed}.
\begin{equation*}
  \|\phi(g^*(x)) - \phi(g_t(x))\|_{\mathcal{H}_Z} \leq L \|g^*(x) - g_t(x)\| \leq L \eta_t.
\end{equation*}

Substituting this bound back into the Cauchy-Schwarz inequality completes the proof.
\end{proof}

We next prove the main theorem, which is that every point in the suboptimal set for the approximated function $f_t$ is in a corresponding suboptimal set for $f$.
\begin{proof}[Proof of Theorem \ref{thm}.]
We decompose the optimality gap of $x_t$ with respect to $f$ using the approximated function $f_t$:
\begin{equation*}
    f(x^*) - f(x_t) = \big(f(x^*) - f_t(x^*)\big) + \big(f_t(x^*) - f_t(x_t)\big) + \big(f_t(x_t) - f(x_t)\big).
\end{equation*}
By Lemma \ref{lem:f_approximation}, the first and third terms are each bounded by $B L \eta_t$. For the middle term, we are considering the optimum of $f$ evaluated by $f_t$, which is bounded by the optimal value of $f_t$ itself:
\begin{equation*}
    f_t(x^*) - f_t(x_t) \leq \max_{x} f_t(x) - f_t(x_t) \leq \delta.
\end{equation*}
Summing these three upper bounds produces the desired result.
\end{proof}

\section{Additional Results}
\label{app:additional-results}

\subsection{Selected Feature Dimensionality}
The main text focuses on whether elicited features improve surrogate fit.
For completeness, Figure~\ref{fig:app-d-chosen} reports the selected feature dimensionality over optimization rounds.
The selected spaces typically contain two to three features, with a mild increasing trend as more prompt-score evidence becomes available.
This supports the intended use of a compact yet increasingly more complex representation for GP modeling under small evaluation budgets.
Despite the increase in embedding dimensionality, those features are becoming iteratively more predictive of observed scores, as shown in Figure~\ref{fig:gp-fit}.

\begin{figure}[ht]
\centering
\includegraphics[width=0.55\textwidth]{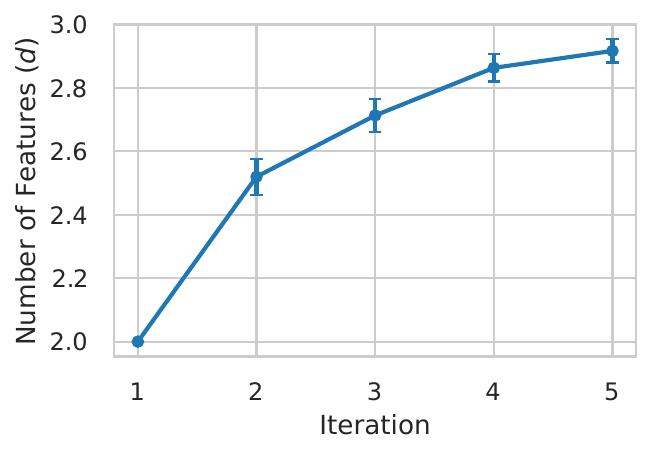}
\caption{Selected feature dimensionality over optimization rounds, reported as mean $\pm$ 95\% CI across tasks and seeds.
The elicited feature spaces typically contain two to three semantic dimensions.}
\label{fig:app-d-chosen}
\end{figure}

\subsection{Refinement Gap Trajectory}
Figure~\ref{fig:app-refinement-gap} reports the best realized-target $\ell_2$ gap over refinement steps.
This diagnostic checks that feature-gap feedback can effectively move generated prompts toward BO-selected feature targets.
As shown in Figure~\ref{fig:gap-vs-improve} from the main text, smaller $\ell_2$ gap is directly associated with the probability of observing an improvement in score in the next iteration.

\begin{figure}[ht]
\centering
\includegraphics[width=0.55\textwidth]{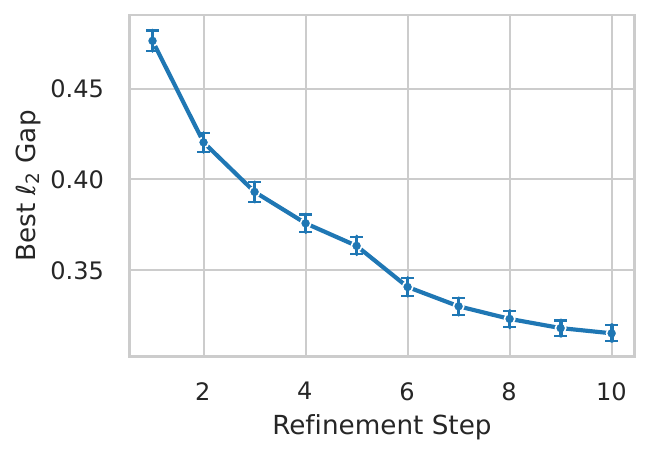}
\caption{Best $\ell_2$ gap between the generated prompt's extracted features and the BO-selected target over refinement steps.
The trajectory shows that the refinement loop moves generated prompts closer to target feature vectors under the allotted refinement budget.}
\label{fig:app-refinement-gap}
\end{figure}

\subsection{Ablation Final-Score Statistics}
\label{app:ablation-final-score-stats}

Figure~\ref{fig:ablation-convergence} shows the full convergence trajectories for the ablation variants.
To quantify the final-budget comparison, Table~\ref{tab:ablation-paired-delta} reports paired final-score differences at $N=30$ evaluations.
For each task and seed, we compute
\[
\Delta_{m,t,s}
=
\mathrm{score}_{m,t,s}
-
\mathrm{score}_{\method,t,s},
\]
where $m$ is an ablation variant.
Negative values therefore favor \method.
We pool the paired differences across task-seed pairs and test whether the mean paired difference is zero.

\begin{table}[h]
\centering
\small
\begin{tabular}{lcc}
\toprule
\textbf{Variant} & \textbf{Pooled paired $\Delta$ vs.\ \method} & \textbf{$p$} \\
\midrule
No Refinement & $-0.009 \pm 0.004$ & $<0.001$ \\
No BO & $-0.007 \pm 0.003$ & $<0.001$ \\
Static Features & $-0.003 \pm 0.004$ & $0.168$ \\
Independent Extraction & $-0.002 \pm 0.004$ & $0.329$ \\
\bottomrule
\end{tabular}
\vspace{2ex}
\caption{Pooled paired final-score differences between each ablation and \method across all task-seed pairs.
Each entry reports $\Delta=\mathrm{ablation}-\method$, so negative values favor \method.
The $\pm$ term denotes a 95\% confidence interval for the mean paired difference, and $p$-values are from two-sided paired $t$-tests.}
\label{tab:ablation-paired-delta}
\end{table}

The results support two main conclusions.
First, removing feature-gap refinement and replacing BO with random feature-space sampling produce the largest and statistically significant drops, indicating that both target realization and acquisition-guided search contribute to final performance.
Second, freezing the feature set and extracting prompts independently yield smaller negative point estimates that are not statistically significant under this pooled test.
Thus, the ablation evidence for dynamic re-elicitation should be interpreted together with the surrogate-fit diagnostic in Figure~\ref{fig:gp-fit}, while the independent-extraction result supports joint extraction as an efficiency improvement that does not appear to harm final performance.

Because this table pools heterogeneous benchmark tasks, the paired test should be interpreted as an aggregate diagnostic rather than a claim that every ablation degrades performance on every task.

\subsection{Feature Evolution Case Study}
Table~\ref{tab:case-study} gives a qualitative example of the tasks evaluated: selected features evolve across iterations while cross-validation MSE decreases.

\begin{table}[ht]
\centering
\scriptsize
\begin{tabular}{r r l p{0.51\linewidth}}
\toprule
\textbf{Iter} & \textbf{MSE} & \textbf{Feature} & \textbf{Description} \\
\midrule[\heavyrulewidth]
\textbf{1} & 0.0612 & explicitness\_of\_sarcasm\_cues & This feature measures the degree to which the prompt explicitly mentions cues for identifying sarcasm, such as 'ironic or mocking', 'tone, context, and figurative language', or 'exaggeration, understatement, and idiomatic expressions'. A value of 0 represents a prompt that does not mention any such cues, while a value of 1 represents a prompt that explicitly mentions multiple relevant cues. This feature causally affects performance because explicitly mentioning these cues guides the AI to focus on the most relevant aspects of the input statements for sarcasm detection. \\
 &  & focus\_on\_pragmatic\_analysis & This feature measures the extent to which the prompt emphasizes the importance of pragmatic analysis, such as considering 'the speaker's intent and the audience's potential interpretation' or 'assessing their semantic meaning and pragmatics'. A value of 0 indicates a prompt that does not emphasize pragmatic analysis, while a value of 1 represents a prompt that clearly directs the AI to consider the pragmatic aspects of the statements. This feature causally affects performance because pragmatic analysis is crucial for distinguishing between literal and sarcastic meanings. \\
\midrule[\heavyrulewidth]
\textbf{3} & 0.0199 & explicitness\_of\_sarcasm\_cues & This feature measures the degree to which the prompt explicitly mentions specific cues for identifying sarcasm, such as exaggeration, understatement, idiomatic expressions, irony, or mocking tone. A value of 0 represents a prompt that does not mention any specific cues, while a value of 1 represents a prompt that mentions multiple relevant cues. This feature causally affects performance because explicit cues guide the AI to focus on the most relevant aspects of the input statements for sarcasm detection. \\
 &  & simplicity\_of\_task\_description & This feature measures the extent to which the prompt provides a simple and direct description of the task, without requiring additional steps or complex analysis. A value of 0 indicates a prompt that requires the AI to perform additional tasks or analysis, while a value of 1 represents a prompt that clearly and directly describes the task. This feature causally affects performance because simplicity reduces ambiguity and guides the AI's analysis. \\
 &  & focus\_on\_literal\_vs\_figurative\_meaning & This feature measures the degree to which the prompt emphasizes the distinction between literal and figurative meanings in identifying sarcasm. A value of 0 represents a prompt that does not emphasize this distinction, while a value of 1 represents a prompt that clearly directs the AI to consider the contrast between literal and figurative meanings. This feature causally affects performance because understanding this distinction is crucial for identifying sarcasm. \\
\midrule[\heavyrulewidth]
\textbf{5} & 0.0119 & emphasis\_on\_literal\_figurative\_contrast & This feature measures the degree to which the prompt emphasizes understanding the contrast between literal and figurative meanings in identifying sarcasm. A value of 0 represents a prompt that does not emphasize this distinction, while a value of 1 represents a prompt that clearly directs the AI to consider the contrast between literal and figurative meanings. This feature causally affects performance because understanding this distinction is crucial for identifying sarcasm. \\
 &  & cue\_explicitness\_and\_brevity & This feature measures the degree to which the prompt explicitly mentions specific cues for identifying sarcasm (such as exaggeration, understatement, idiomatic expressions, irony, or mocking tone) in a concise manner. A value of 0 represents a prompt that either does not mention any specific cues or does so in a verbose or indirect way, while a value of 1 represents a prompt that mentions relevant cues clearly and briefly. This feature causally affects performance because explicit and concise cues guide the AI to focus on the most relevant aspects of the input statements for sarcasm detection without introducing unnecessary complexity. \\
 &  & avoidance\_of\_overcontextualization & This feature measures the degree to which the prompt avoids requiring the AI to heavily consider the broader context, speaker's intent, or audience interpretation beyond the explicit language and tone used. A value of 0 represents a prompt that emphasizes the importance of these contextual factors, while a value of 1 represents a prompt that focuses primarily on the statements themselves and the linguistic cues they contain. This feature causally affects performance because overcontextualization can lead to unnecessary complexity and ambiguity in sarcasm detection. \\
\bottomrule
\end{tabular}
\caption{Case study: features the LLM proposes at iterations 1, 3, 5 on Snarks. Dimensionality grows from $d=2$ at iteration~1 to $d=3$ at iteration 3 and 5.
GP LOO-CV MSE decreases each iteration as we continuously iterate and regenerate the features.}
\label{tab:case-study}
\end{table}

\section{Implementation Details}

\subsection{Algorithms}

\subsubsection{Initial Dataset Generation}

\begin{algorithm}[h]
\caption{Initial Dataset $\mathcal{D}_0$ Generation}
\label{alg:d0}

\KwIn{Task context $c$, batch size $q$, seed}
\KwOut{$\mathcal{D}_0 = \{(x_i, y_i)\}_{i=1}^q$}

$\{x_1, \ldots, x_q\} \leftarrow \text{LLM}(c, q)$
\Comment*[f]{$\mathcal{D}_0$ generation prompt: generate $q$ diverse system prompts}
\For{$i = 1, \ldots, q$}{
  $y_i \leftarrow f(x_i)$
  \Comment*[f]{evaluate sequentially}
}
$\mathcal{D}_0 \leftarrow \{(x_i, y_i)\}_{i=1}^q$\;
\KwReturn{$\mathcal{D}_0$}\;
\end{algorithm}

\subsubsection{\method Algorithm}
\label{sec:full_algo}

Algorithm~\ref{alg:reelicit-full} described full \method algorithm.

\begin{algorithm}[h]
\small
\caption{\method. Expanded version of Algorithm~\ref{alg:reelicit}.}
\label{alg:reelicit-full}

\KwIn{$\mathcal{D}_0 = \{(x_{0,j}, y_{0,j})\}_{j=1}^q$, total evaluated batches $T$, elicitation rounds $K$, batch size $q$, extraction batch size $b$, realization budget $M \geq 1$, tolerance $\tau$}
\KwOut{$x^* = \arg\max_{(x,y) \in \mathcal{D}_{T-1}} y$}

$\mathcal{F}_0 \leftarrow \varnothing$\;

\For{$t = 1, \ldots, T-1$}{

  Let $X_{t-1},Y_{t-1}$ be the prompts and scores in $\mathcal{D}_{t-1}$\;

  \BlankLine
  \Comment{Phase 1: Elicit features and build embedding $g_t$}
  \ParFor{$k = 1, \ldots, K$}{
    $\mathcal{F}_t^{(k)} \leftarrow \DefineFeatures(\mathcal{D}_{t-1}, \mathcal{F}_{t-1})$\;
    \Comment{\textsc{DefineFeatures} prompt}

    $Z_t^{(k)} \leftarrow \ExtractFeatures(X_{t-1}, \mathcal{F}_t^{(k)}, b)$\;
    \Comment{\textsc{ExtractFeatures} prompt; $z_i = g_{\mathcal{F}_t^{(k)}}(x_i)$; scores are not shown}

    $\text{score}_t^{(k)} \leftarrow \text{CV}(Z_t^{(k)}, Y_{t-1})$\;
  }

  $\mathcal{K}_t \leftarrow \{1,\ldots,K\}$\;

  \If{$t > 1$}{
    $\mathcal{F}_t^{(K+1)} \leftarrow \mathcal{F}_{t-1}$\;
    $Z_t^{(K+1)} \leftarrow \ExtractFeatures(X_{t-1}, \mathcal{F}_{t-1}, b)$\;
    $\text{score}_t^{(K+1)} \leftarrow \text{CV}(Z_t^{(K+1)}, Y_{t-1})$\;
    $\mathcal{K}_t \leftarrow \mathcal{K}_t \cup \{K+1\}$\;
    \Comment{Re-score the incumbent feature set on the current history}
  }

  $k^* \leftarrow \arg\min_{k \in \mathcal{K}_t} \text{score}_t^{(k)}$\;
  $\mathcal{F}_t \leftarrow \mathcal{F}_t^{(k^*)}$,\quad $Z_t \leftarrow Z_t^{(k^*)}$\;

  \BlankLine
  \Comment{Phase 2: Bayesian optimization in embedding space}
  Fit GP $\mathcal{M}_t$ on $(Z_t, Y_{t-1})$\;

  $\{z_{t,1}^{\mathrm{new}}, \ldots, z_{t,q}^{\mathrm{new}}\}
  \leftarrow
  \arg\max_{z_1,\ldots,z_q \in [0,1]^{d_t}}
  \alpha(z_1,\ldots,z_q \mid \mathcal{M}_t)$\;
  \Comment{BO selects feature targets, not text prompts}

  \BlankLine
  \Comment{Phase 3: Realize feature targets as prompts}
  \ParFor{$j = 1, \ldots, q$}{
    $M_{\mathrm{init}} \leftarrow \max(1, \lfloor M/2 \rfloor)$;\quad
    $M_{\mathrm{refine}} \leftarrow M - M_{\mathrm{init}}$\;

    \BlankLine
    \Comment{3a: Parallel initial generation}
    \ParFor{$p = 1, \ldots, M_{\mathrm{init}}$}{
      $\tilde{\mathcal{D}}_{p} \leftarrow \StratSample(\mathcal{D}_{t-1}, n_{\max})$\;
      $x^{(p)} \leftarrow \textsc{InitialGenerate}(\tilde{\mathcal{D}}_{p}, \mathcal{F}_t, z_{t,j}^{\mathrm{new}})$\;
      $z^{(p)} \leftarrow \ExtractFeatures(\{x^{(p)}\}, \mathcal{F}_t, b{=}1)$\;
    }

    $p^* \leftarrow \arg\min_p \|z_{t,j}^{\mathrm{new}} - z^{(p)}\|_2$\;
    $x_{\mathrm{best}}, z_{\mathrm{best}} \leftarrow x^{(p^*)}, z^{(p^*)}$\;

    \BlankLine
    \Comment{3b: Sequential refinement by feature-gap reduction}
    \For{$i = 1, \ldots, M_{\mathrm{refine}}$}{
      \lIf{$\|z_{t,j}^{\mathrm{new}} - z_{\mathrm{best}}\|_2 \leq \tau$}{\Break}

      $\Delta_\ell \leftarrow (z_{t,j}^{\mathrm{new}})_\ell - (z_{\mathrm{best}})_\ell$ for $\ell = 1, \ldots, d_t$\;
      Sort features by $|\Delta_\ell|$ descending\;

      $x_{\mathrm{new}} \leftarrow \textsc{FeatureGuidedRefine}(x_{\mathrm{best}}, \mathcal{F}_t, z_{t,j}^{\mathrm{new}}, z_{\mathrm{best}})$\;
      $z_{\mathrm{new}} \leftarrow \ExtractFeatures(\{x_{\mathrm{new}}\}, \mathcal{F}_t, b{=}1)$\;

      \lIf{$\|z_{t,j}^{\mathrm{new}} - z_{\mathrm{new}}\|_2 < \|z_{t,j}^{\mathrm{new}} - z_{\mathrm{best}}\|_2$}{
        $x_{\mathrm{best}}, z_{\mathrm{best}} \leftarrow x_{\mathrm{new}}, z_{\mathrm{new}}$
      }
    }

    $x_{t,j}^{\mathrm{new}} \leftarrow x_{\mathrm{best}}$\;
  }

  \BlankLine
  \Comment{Phase 4: Evaluate and update}
  $y_{t,j}^{\mathrm{new}} \leftarrow f(x_{t,j}^{\mathrm{new}})$ for each $j$\;
  $\mathcal{D}_{t}
  \leftarrow
  \mathcal{D}_{t-1}
  \cup
  \{(x_{t,j}^{\mathrm{new}}, y_{t,j}^{\mathrm{new}})\}_{j=1}^{q}$\;
}
\end{algorithm}

\subsubsection{Baselines}
\label{sec:baselines_details}

\paragraph{APE-style guided sampling} (Algorithm~\ref{alg:ape}, Section~\ref{prompt:ape}).
This is a history-free candidate-generation baseline that captures the core idea of APE: candidate generation plus evaluation.
At each iteration, the optimizer LLM generates $q$ diverse system prompts from the task description alone, without using any optimization history.
This is a reasonable simplification of APE, though not a full reproduction of the original search protocol.

\paragraph{OPRO} (Algorithm~\ref{alg:opro}, Section~\ref{prompt:opro}).
The main iterative baseline. At each iteration, OPRO presents a stratified subsample of the $(x, y)$ history sorted worst-to-best (placing the best prompts last for recency bias) and asks the optimizer LLM to generate $q$ improvements.
OPRO is the closest native fit among the evaluated baselines because it uses solution-score histories to generate improved candidates.

\paragraph{PromptBreeder} (Algorithm~\ref{alg:pb}, Sections~\ref{prompt:pb-mutation} and~\ref{prompt:pb-recombination}).
A population-based prompt evolution baseline inspired by PromptBreeder.
Maintains a fitness-sorted population (default size 20). Each iteration generates $q$ offspring: $q{-}1$ via mutation (one of three operators: improve clarity, make reasoning explicit, increase conciseness) and 1 via recombination of two random parents.
All $q$ LLM calls run concurrently. This is a natural match for the black-box setting because evolutionary methods require only fitness values.

\paragraph{TextGrad-style black-box refinement} (Algorithm~\ref{alg:tg}, Section~\ref{prompt:textgrad}).
A critique-then-improve baseline inspired by TextGrad.
Presents a stratified trajectory sample plus the current best prompt, then follows a 3-step chain: (1) analyze trajectory patterns, (2) critique the best prompt, (3) generate $q$ improved variants addressing different critique aspects.
This adaptation requires the largest interface change because original TextGrad benefits from richer structured feedback, such as instance-level errors or textual gradients. We therefore label it “TextGrad-style” to be transparent about the aggregate-only simplification.

\paragraph{Shared infrastructure.}
The stratified subsampling utility is shared by \method (\textsc{DefineFeatures}, \textsc{GenerateWithRefinement}), OPRO, and TextGrad, ensuring all methods see the same in-context distribution.
When history exceeds $n_{\max}$, it draws the top 25\% by score, bottom 25\%, and a random sample of the middle 50\%.
OPRO and TextGrad sort the subsample ascending by score (best last) for recency bias.

\begin{algorithm}[h]
\caption{APE-style Guided Sampling}
\label{alg:ape}

\KwIn{Initial evaluated dataset $\mathcal{D}_0$, total evaluated batches $T$, task context $c$, batch size $q$}
\KwOut{Best prompt $x^*$}

\For{$t = 1, \ldots, T-1$}{
  $\{x_{t,1}^{\mathrm{new}}, \ldots, x_{t,q}^{\mathrm{new}}\} \leftarrow \text{LLM}(c, q)$
  \Comment*[f]{APE prompt; history is not used}

  Evaluate $y_{t,j}^{\mathrm{new}} \leftarrow f(x_{t,j}^{\mathrm{new}})$ for $j = 1, \ldots, q$\;

  $\mathcal{D}_{t}
  \leftarrow
  \mathcal{D}_{t-1}
  \cup
  \{(x_{t,j}^{\mathrm{new}}, y_{t,j}^{\mathrm{new}})\}_{j=1}^{q}$\;
}

\KwReturn{$x^* = \arg\max_{(x,y) \in \mathcal{D}_{T-1}} y$}\;
\end{algorithm}

\begin{algorithm}[h]
\caption{OPRO: History-Conditioned Prompt Generation}
\label{alg:opro}

\KwIn{Initial evaluated dataset $\mathcal{D}_0$, total evaluated batches $T$, task context $c$, batch size $q$, context cap $n_{\max}$}
\KwOut{Best prompt $x^*$}

\For{$t = 1, \ldots, T-1$}{
  $\tilde{\mathcal{D}}_{t-1} \leftarrow \StratSample(\mathcal{D}_{t-1}, n_{\max})$\;

  Sort $\tilde{\mathcal{D}}_{t-1}$ by score ascending (worst $\to$ best)\;

  $\{x_{t,1}^{\mathrm{new}}, \ldots, x_{t,q}^{\mathrm{new}}\}
  \leftarrow
  \text{LLM}(c, \tilde{\mathcal{D}}_{t-1}, q)$
  \Comment*[f]{OPRO prompt}

  Evaluate $y_{t,j}^{\mathrm{new}} \leftarrow f(x_{t,j}^{\mathrm{new}})$ for $j = 1, \ldots, q$\;

  $\mathcal{D}_{t}
  \leftarrow
  \mathcal{D}_{t-1}
  \cup
  \{(x_{t,j}^{\mathrm{new}}, y_{t,j}^{\mathrm{new}})\}_{j=1}^{q}$\;
}

\KwReturn{$x^* = \arg\max_{(x,y) \in \mathcal{D}_{T-1}} y$}\;
\end{algorithm}

\begin{algorithm}[h]
\caption{PromptBreeder: Population-Based Prompt Evolution}
\label{alg:pb}

\KwIn{Initial evaluated dataset $\mathcal{D}_0$, total evaluated batches $T$, task context $c$, batch size $q$, population size $P_{\max}$}
\KwOut{Best prompt $x^*$}

\For{$t = 1, \ldots, T-1$}{
  Let $\mathcal{P}_{t-1}$ be the top-$P_{\max}$ entries of $\mathcal{D}_{t-1}$ sorted by score descending\;

  \ParFor{$j = 1, \ldots, q$}{
    \eIf{$j < q$}{
      $x_{\mathrm{parent}} \leftarrow$ random choice from $\mathcal{P}_{t-1}$\;

      $m \leftarrow$ random choice from \{clarity, reasoning, conciseness\}\;

      $x_{t,j}^{\mathrm{new}} \leftarrow \text{LLM}(c, x_{\mathrm{parent}}, m)$
      \Comment*[f]{PromptBreeder mutation prompt}
    }{
      $x_{p_1}, x_{p_2} \leftarrow$ random pair from $\mathcal{P}_{t-1}$\;

      $x_{t,j}^{\mathrm{new}} \leftarrow \text{LLM}(c, x_{p_1}, x_{p_2})$
      \Comment*[f]{PromptBreeder recombination prompt}
    }
  }

  Evaluate $y_{t,j}^{\mathrm{new}} \leftarrow f(x_{t,j}^{\mathrm{new}})$ for $j = 1, \ldots, q$\;

  $\mathcal{D}_{t}
  \leftarrow
  \mathcal{D}_{t-1}
  \cup
  \{(x_{t,j}^{\mathrm{new}}, y_{t,j}^{\mathrm{new}})\}_{j=1}^{q}$\;
}

\KwReturn{$x^* = \arg\max_{(x,y) \in \mathcal{D}_{T-1}} y$}\;
\end{algorithm}

\begin{algorithm}[h]
\caption{TextGrad-style Black-Box Refinement}
\label{alg:tg}

\KwIn{Initial evaluated dataset $\mathcal{D}_0$, total evaluated batches $T$, task context $c$, batch size $q$, context cap $n_{\max}$}
\KwOut{Best prompt $x^*$}

\For{$t = 1, \ldots, T-1$}{
  $\tilde{\mathcal{D}}_{t-1} \leftarrow \StratSample(\mathcal{D}_{t-1}, n_{\max})$\;

  Sort $\tilde{\mathcal{D}}_{t-1}$ by score ascending (worst $\to$ best)\;

  $(x_{\mathrm{best}}, y_{\mathrm{best}})
  \leftarrow
  \arg\max_{(x,y) \in \mathcal{D}_{t-1}} y$
  \Comment*[f]{from full history}

  $\{x_{t,1}^{\mathrm{new}}, \ldots, x_{t,q}^{\mathrm{new}}\}
  \leftarrow
  \text{LLM}(c, \tilde{\mathcal{D}}_{t-1}, x_{\mathrm{best}}, y_{\mathrm{best}}, q)$
  \Comment*[f]{TextGrad-style prompt: analyze, critique, generate variants}

  Evaluate $y_{t,j}^{\mathrm{new}} \leftarrow f(x_{t,j}^{\mathrm{new}})$ for $j = 1, \ldots, q$\;

  $\mathcal{D}_{t}
  \leftarrow
  \mathcal{D}_{t-1}
  \cup
  \{(x_{t,j}^{\mathrm{new}}, y_{t,j}^{\mathrm{new}})\}_{j=1}^{q}$\;
}

\KwReturn{$x^* = \arg\max_{(x,y) \in \mathcal{D}_{T-1}} y$}\;
\end{algorithm}

\subsection{Benchmarks}
\label{sec:benchmarks_details}

Table~\ref{tab:tasks-full} describes each task, the number of questions used, and the task context supplied to the LLMs.

\begin{table}[h]
\centering
\small
\begin{tabularx}{\textwidth}{lrX}
\toprule
\textbf{Task} & \textbf{Size} & \textbf{Task Context} \\
\midrule

GSM8K & 500 &
The AI assistant receives a grade-school math word problem requiring multi-step arithmetic reasoning. It must produce a step-by-step solution ending with the final numerical answer. Performance is measured by exact-match accuracy on the final answer. \\

\midrule
MMLU & 500 &
The AI assistant receives a multiple-choice knowledge question drawn from diverse academic subjects including STEM, humanities, social sciences, and professional domains. It must select the correct answer from 4 options (A--D). Performance is measured by accuracy. \\

\midrule
Boolean Expressions & 250 &
The AI assistant receives a boolean expression composed of True/False values connected by logical operators (and, or, not) and must evaluate the expression to determine whether the result is True or False. Performance is measured by accuracy. \\

\midrule
Causal Judgement & 250 &
The AI assistant receives a scenario describing events and outcomes and must determine whether a specific factor was the cause of the outcome, answering Yes or No. This tests causal reasoning and counterfactual thinking. Performance is measured by accuracy. \\

\midrule
Disambiguation QA & 250 &
The AI assistant receives a sentence with an ambiguous pronoun and must determine which entity the pronoun refers to, selecting from multiple choices. This tests coreference resolution and language understanding. Performance is measured by accuracy. \\

\midrule
Formal Fallacies & 250 &
The AI assistant receives a deductive argument consisting of premises and a conclusion and must determine whether the argument is logically valid or invalid. This tests formal logical reasoning. Performance is measured by accuracy. \\

\midrule
Hyperbaton & 250 &
The AI assistant receives two sentences that differ only in adjective ordering and must determine which sentence has the correct (natural) English adjective order. This tests knowledge of English adjective ordering conventions. Performance is measured by accuracy. \\

\midrule
Penguins in a Table & 250 &
The AI assistant receives a table of penguin attributes (name, age, height, weight) and a question about the data. It must parse the table, reason about the attributes, and select the correct answer from multiple choices. This tests structured data parsing and tabular reasoning. Performance is measured by accuracy. \\

\midrule
Snarks & 250 &
The AI assistant receives two nearly identical statements and must determine which one is sarcastic. This tests understanding of pragmatics, tone, and the distinction between literal and figurative meaning. Performance is measured by accuracy. \\

\midrule
Tracking Shuffled Objects & 250 &
The AI assistant receives a description of objects being swapped between people in a series of exchanges. It must track the final position of each object and select the correct answer from multiple choices. Performance is measured by accuracy. \\

\bottomrule
\end{tabularx}
\caption{Full task context strings. Each string is injected into all optimizer and baseline prompts as the \tvar{task\_context} variable. GSM8K and MMLU use fixed 500-question subsamples; the remaining eight tasks are from BIG-Bench Hard (BBH, 250 examples each).}
\label{tab:tasks-full}
\end{table}

\subsection{Evaluation Protocol and Implementation Details}
\label{app:eval_protocol}

\subsubsection{Models and benchmark evaluation.}
\label{app:lms}
We use Llama 3.3 70B Instruct as the optimizer LLM and Llama 3.1 8B Instruct as the target LLM; both have 128K-token context windows.
The optimizer LLM performs feature elicitation, feature extraction, prompt generation, refinement, and baseline candidate generation.
The target LLM is the model whose system prompt is optimized and is accessed only through the evaluation function $f$.
We use the smaller target model to avoid task saturation and reduce evaluation wallclock.

Each evaluation $f(x)$ runs the target LLM with system prompt $x$ on a fixed task subset using the lm-eval framework.
All target-model evaluations are zero-shot, use greedy decoding with \texttt{temperature: 0.0} and \texttt{do\_sample: false}, and inject the candidate system prompt $x$ through the model wrapper.
Optimizer-side LLM calls use temperature 0.7 unless otherwise specified.

\begin{table}[h]
\centering
\small
\begin{tabularx}{\textwidth}{
  >{\raggedright\arraybackslash}p{0.16\textwidth}
  >{\raggedright\arraybackslash}p{0.24\textwidth}
  >{\raggedright\arraybackslash}X
  >{\raggedright\arraybackslash}p{0.24\textwidth}
}
\toprule
\textbf{Format} & \textbf{Task(s)} & \textbf{Input template} & \textbf{Answer extraction} \\
\midrule
Numeric answer
& GSM8K
& \texttt{Question: \{question\}\textbackslash nAnswer:}
& Regex: extract final number \\

Four-way multiple choice
& MMLU
& \texttt{Question: \{question\}\textbackslash n(A)...(D)...\textbackslash nAnswer:}
& Regex: \texttt{\textbackslash b([A-D])\textbackslash b} \\

True/False
& Boolean Expressions
& \texttt{Q: \{input\}\textbackslash nAnswer with just True or False.\textbackslash nA:}
& Regex: \texttt{(true\textbar{}false)}, case-insensitive \\

Yes/No
& Causal Judgement
& \texttt{Q: \{input\}\textbackslash nAnswer with just Yes or No.\textbackslash nA:}
& Regex: \texttt{(yes\textbar{}no)}, case-insensitive \\

Valid/Invalid
& Formal Fallacies
& \texttt{Q: \{input\}\textbackslash nAnswer with just valid or invalid.\textbackslash nA:}
& Regex: \texttt{(valid\textbar{}invalid)}, case-insensitive \\

Letter choice
& Disambiguation QA, Hyperbaton, Penguins in a Table, Snarks, Tracking Shuffled Objects
& \texttt{Q: \{input\}\textbackslash nAnswer with just the letter choice, e.g.\ (A).\textbackslash nA:}
& Regex: \texttt{\textbackslash(([A-F])\textbackslash)} \\
\bottomrule
\end{tabularx}
\caption{Evaluation templates and answer extraction rules, grouped by output format.}
\label{tab:answer-extraction}
\end{table}

\subsubsection{BO surrogate and feature-set selection.}
For \method, the surrogate model is a BoTorch \texttt{SingleTaskGP} with the default Mat\'ern 5/2 kernel with ARD, \texttt{Normalize(d=d)} input transform, and \texttt{Standardize(m=1)} outcome transform.
We fit the GP by maximum likelihood using \texttt{fit\_gpytorch\_mll} with \texttt{ExactMarginalLogLikelihood}.
The acquisition function is \texttt{qLogNoisyExpectedImprovement} with \texttt{X\_baseline} set to the training inputs, optimized over $[0,1]^d$ using \texttt{optimize\_acqf} with 20 restarts and 512 raw samples.
The GP is refit from scratch at each BO iteration on the current prompt-score history.

At each iteration, $K$ independently elicited feature sets compete via held-out prediction error.
We use leave-one-out cross-validation when the number of data points is below 10 and 10-fold cross-validation otherwise.
The selection metric is mean squared error on held-out points.
When $t>1$, the incumbent feature set from the previous iteration is also included as an additional candidate by re-extracting it on the current enlarged history and scoring it by the same cross-validation procedure.
This allows a previously predictive representation to persist, but does not assume that feature quality monotonically improves.
As a diagnostic, we compare GP cross-validation MSE against a constant predictor that outputs $\bar{y}_{\mathrm{train}}$ on held-out points.

\paragraph{In-context history subsampling.}
\label{app:subsampling}
When an optimizer-side LLM call conditions on prompt-score history and the history exceeds $n_{\max}$, we use a stratified subsample:
top 25\% by score, bottom 25\% by score, and a random sample from the middle 50\% to fill the remaining slots, with at least one example from the top and bottom groups.
This utility is shared by \method, OPRO, and TextGrad-style refinement, ensuring that these methods see the same in-context distribution when history is subsampled.

\subsubsection{Default hyperparameters.}
\label{app:hyperparams}
Table~\ref{tab:hyperparams} describes hyperparameters used in experiments.

\begin{table}[h]
\centering
\small
\begin{tabular}{lcl}
\toprule
\textbf{Parameter} & \textbf{Value} & \textbf{Description} \\
\midrule
$N$ & 30 & Total evaluation budget including $\mathcal{D}_0$ \\
$q$ & 5 & Candidates per iteration \\
$T$ & 6 & Total evaluated batches: $t{=}0$ is $\mathcal{D}_0$, and $t{=}1,\ldots,5$ are optimization rounds \\
$K$ & 5 & Feature elicitation rounds per iteration \\
$M$ & 10 & Total generation and refinement budget per candidate \\
$\tau$ & 0.1 & $\ell_2$ distance tolerance for refinement early stopping \\
$b$ & 10 & Feature-extraction batch size \\
$n_{\max}$ & 12 & Maximum in-context examples for optimizer-side LLM calls \\
$P_{\max}$ & 20 & PromptBreeder population cap \\
\bottomrule
\end{tabular}
\caption{Default hyperparameters.}
\label{tab:hyperparams}
\end{table}

\subsubsection{Ablation variants.}
\label{app:ablations}

Table~\ref{tab:ablations} describes variants tested in ablation study.

\begin{table}[h]
\centering
\small
\begin{tabularx}{\textwidth}{lX}
\toprule
\textbf{Variant} & \textbf{What changes} \\
\midrule
No refinement
& Skip sequential feature-gap refinement and use the initial realized prompt for each BO-selected feature target. \\

No BO
& Replace acquisition optimization with uniform random sampling in $[0,1]^d$. \\

Static features
& Freeze the first selected feature set and reuse it in subsequent optimization rounds without re-elicitation. \\

Independent extraction
& Set $b=1$, extracting features for each prompt independently instead of jointly in batches. \\
\bottomrule
\end{tabularx}
\caption{\method ablation variants. Each variant changes one design choice.}
\label{tab:ablations}
\end{table}

\subsubsection{Information access control.}
\label{app:info-access}
The optimizer LLM is used by several subroutines with different information access.
This separation is important because feature definitions may depend on prompt-score history, but extracted feature values should be based on prompt content rather than direct access to outcomes.

\begin{table}[h]
\centering
\small
\begin{tabular}{lcccc}
\toprule
\textbf{Subroutine} & \textbf{Prompts $x_i$} & \textbf{Scores $y_i$} & \textbf{Features $\mathcal{F}$} & \textbf{Target $z^{\mathrm{target}}$} \\
\midrule
\textsc{DefineFeatures} & \checkmark & \checkmark & \checkmark\textsuperscript{*} & --- \\
\textsc{ExtractFeatures} & \checkmark & --- & \checkmark & --- \\
\textsc{InitialGenerate} & \checkmark & \checkmark & \checkmark & \checkmark \\
\textsc{FeatureGuidedRefine} & \checkmark & \checkmark & \checkmark & \checkmark \\
\bottomrule
\end{tabular}
\caption{Information available to optimizer-LLM subroutines.}
\label{tab:info-access}
\end{table}

\noindent\textsuperscript{*}When $t>1$, \textsc{DefineFeatures} also sees the incumbent feature set $\mathcal{F}_{t-1}$.

\textsc{ExtractFeatures} never sees evaluation scores.
Thus, although scores are used to choose which feature set is most predictive, the coordinates assigned to prompts are extracted from prompt content under the selected feature definitions rather than copied or inferred directly from outcomes.

\subsection{Analysis Details}

For completeness, we define the linear CKA used in Figure~\ref{fig:feature-stability}.
Given two representations $Z \in \mathbb{R}^{n \times d_Z}$ and $Z' \in \mathbb{R}^{n \times d_{Z'}}$ evaluated on the same $n$ prompts, let $K = ZZ^\top$ and $L = Z'Z'^\top$ be their linear Gram matrices.
Linear CKA is
\[
\mathrm{CKA}(K,L)
=
\frac{\langle HKH, HLH\rangle_F}
{\|HKH\|_F \|HLH\|_F},
\quad
H = I_n - \frac{1}{n}\mathbf{1}\mathbf{1}^\top .
\]
Here $\|\cdot\|_F$ and $\langle \cdot,\cdot\rangle_F$ denote the Frobenius norm and inner product.

\subsection{Existing assets and licenses.}
We use public benchmark datasets GSM8K, MMLU, and BIG-Bench Hard, and cite their original papers in the main text.
GSM8K, MMLU, and the BIG-Bench Hard repository are distributed under MIT licenses.
We use the lm-evaluation-harness and BoTorch software packages, both under MIT licenses.
The optimizer and target models are Llama 3.3 70B Instruct and Llama 3.1 8B Instruct, used under their respective Meta Llama Community License Agreements and Acceptable Use Policies.
We do not redistribute model weights or modified benchmark datasets.

\clearpage
\section{Prompts}
\label{app:prompts}

\subsection{\method Prompts}

\subsubsection{\textsc{DefineFeatures}}
\label{prompt:define-features}

{\small
This is the \textbf{only} prompt where the optimizer LLM sees evaluation scores $y_i$.

\begin{promptbox}[breakable, title={\textsc{DefineFeatures}}]
You are an expert at analyzing text objects and identifying patterns that predict performance.

These text objects are system prompts for an AI assistant performing the following task:
\tvar{task\_context}

Define numerical features that capture what makes a prompt effective FOR THIS SPECIFIC TASK. Focus on properties that have a CAUSAL relationship to the AI's ability to solve this type of problem --- properties that, if changed in a prompt, would directly affect performance. Avoid features that merely correlate with score or describe surface-level text properties without a causal mechanism.

Before defining features, closely inspect the data:
\begin{enumerate}
\item What specific patterns or properties are present in the TOP-performing text objects but absent from the BOTTOM ones?
\item What do the BOTTOM-performing text objects have that the TOP ones don't?
\item List 2--3 concrete observations about these differences.
\end{enumerate}
Then define features that capture these observed causal differences.

Requirements for each feature:
\begin{itemize}
\item \textbf{name}: A short identifier.
\item \textbf{description}: Explain what the feature measures for this specific task, what 0 and 1 represent (anchor semantics), and why it causally affects performance.
\item All feature values must be in [0, 1].
\item Each feature MUST be INDEPENDENT of the others. If you can predict one feature's value from the others, they are redundant --- keep only the more causally relevant one.
\item Before finalizing, verify: for each pair of features, can you imagine a text that is high on one and low on the other? If not, they are not independent --- drop one.
\end{itemize}

Choose the number of features based on the available data. With small datasets (10 or fewer examples), prefer fewer features (1--2) --- a single well-chosen feature is better than several noisy ones. As the dataset grows and patterns become clearer, additional features can capture richer structure that was not visible with less data. Every feature must earn its place by capturing genuinely independent variation.

Respond with a JSON array of objects, each with `name' and `description' fields. Example:
\begin{verbatim}
[
  {
    "name": "example_feature_name",
    "description": "What this feature measures for the specific task,"\
            "what 0 and 1 represent, and why it causally affects performance."
  }
]
\end{verbatim}

\emph{[When $t > 1$ and incumbent features $\mathcal{F}_{t-1}$ are provided, the following is prepended to the input. Incumbent features are shown in shuffled order to avoid position bias.]}

The following features are currently in use for this task:\\
--- \tvar{f.name}: \tvar{f.description} \quad\emph{[for each incumbent feature]}

Analyze the data to identify patterns or properties NOT captured by these features. You may keep features that are clearly the most predictive, but your proposed set MUST differ from the current set by at least one meaningful change: add a genuinely new feature, remove a feature, or substantively redefine one. Renaming or trivially rephrasing an existing feature does NOT count as a change.

\medskip
\emph{[Then always:]}
\medskip

Here are \tvar{n} text objects with their performance scores (higher is better), sorted by performance tier:

\tvar{examples\_text} \quad\emph{[formatted with TOP/BOTTOM tier labels]}

You have \tvar{n} text objects available. Be mindful of this sample size --- fewer high-quality, genuinely distinct features are far better than many overlapping ones.

First list your observations about what distinguishes TOP from BOTTOM performers, then define features. Return ONLY the JSON array.
\end{promptbox}
}

\subsubsection{\textsc{ExtractFeatures}}
\label{prompt:extract-features}

{\small
Does \textbf{not} include evaluation scores $y_i$ --- prevents information leakage.

\begin{promptbox}[breakable, title={\textsc{ExtractFeatures}}]
You are an expert at analyzing text and rating it on specific features. For each text object, assign a value in [0, 1] for each feature based on the feature description.

These text objects are system prompts for an AI assistant performing the following task:
\tvar{task\_context}

Rate each text object considering how the features relate to this specific task.

Be consistent: similar texts should get similar scores. Use the full range of [0, 1] -- don't cluster all values near the middle.

Respond with a JSON object keyed by text object ID, where each value is an object mapping feature names to numeric values.

Example:
\begin{verbatim}
{
  "0": {"feature_a": 0.75, "feature_b": 0.30},
  "1": {"feature_a": 0.45, "feature_b": 0.80}
}
\end{verbatim}

Features to rate:\\
--- \textbf{\tvar{f.name}}: \tvar{f.description} \quad\emph{[for each feature]}

Text objects to rate:\\
--- Text Object ID: ``\tvar{tid}'' ---\\
\tvar{content} \quad\emph{[for each text in batch]}

Rate each text object on each feature. Values must be numbers in [0, 1]. Return ONLY the JSON object.
\end{promptbox}
}

\subsubsection{Initial Generation}
\label{prompt:generation}

{\small
Used in Phase A of \textsc{GenerateWithRefinement} (Algorithm~\ref{alg:reelicit-full}).

\begin{promptbox}[breakable, title={Initial Generation}]
You are an expert prompt engineer. Generate a system prompt for an AI assistant performing the following task:
\tvar{task\_context}

The prompt should match specific target feature values.

You will be given:
\begin{enumerate}
\item Feature definitions with their semantics.
\item Example prompts labeled [TOP] or [BOTTOM] by performance, with their feature values and scores.
\item A target feature vector to aim for.
\end{enumerate}

Study what makes the TOP-scoring examples effective and what makes the BOTTOM-scoring examples less effective. Learn from the best examples --- understand the patterns and approaches that lead to high performance.

The target feature vector indicates a promising direction to explore. Generate a NEW system prompt that combines the successful patterns from the TOP examples while matching the target feature values.

Output ONLY the generated prompt text, with no additional commentary or formatting.

\medskip

Feature definitions:\\
--- \textbf{\tvar{f.name}}: \tvar{f.description} \quad\emph{[for each feature]}

Example prompts (sorted by performance, with tier labels):\\
\tvar{examples\_text} \quad\emph{[TOP/BOTTOM labels, features, and scores]}

Target feature vector:\\
\tvar{target\_text} \quad\emph{[JSON dict, e.g.\ \{"conciseness": 0.85, "step\_guidance": 0.70\}]}

Generate a system prompt that combines the best patterns from the TOP examples while matching the target features. Output ONLY the prompt text.
\end{promptbox}
}

\subsubsection{\textsc{FeatureGuidedRefine}}
\label{prompt:refinement}

{\small
Used in Phase B of \textsc{GenerateWithRefinement} (Algorithm~\ref{alg:reelicit-full}).

\begin{promptbox}[breakable, title={\textsc{FeatureGuidedRefine}}]
You are an expert prompt engineer. Modify the given system prompt to better match target feature values.

The system prompt is for an AI assistant performing the following task:
\tvar{task\_context}

Consider what text patterns in the reference examples correspond to the desired feature values, and what specific phrases in the current prompt are causing the gaps.

Rules:
\begin{itemize}
\item Focus on the LARGEST gaps first (they are listed in order of priority).
\item MODIFY the existing text --- do not rewrite from scratch.
\item PRESERVE aspects that are already well-aligned with their targets.
\item Output ONLY the modified prompt text, with no additional commentary or formatting.
\end{itemize}

\emph{[When reference examples are provided:]}

Reference examples (sorted by performance):\\
\tvar{examples\_text} \quad\emph{[TOP/BOTTOM labels, features, and scores]}

Use the TOP examples as reference for the style and patterns that correspond to the desired feature values.

\medskip
\emph{[Then always:]}
\medskip

Current system prompt:\\
\tvar{text}

Feature gap analysis (sorted by gap magnitude, largest first):\\
\tvar{gap\_text} \quad\emph{[JSON array; format shown below]}

Modify the system prompt to reduce the largest feature gaps. Output ONLY the modified prompt text.

\medskip
\emph{[Gap analysis format:]}
\begin{verbatim}
[
  {
    "feature_name": "step_by_step_guidance",
    "definition": "How explicitly the prompt instructs...",
    "target": 0.85,
    "current": 0.3,
    "gap": 0.55,
    "direction": "increase"
  },
  ...
]
\end{verbatim}
\end{promptbox}
}

\clearpage
\subsection{Baseline Prompts}

\subsubsection{APE-style Guided Sampling}
\label{prompt:ape}

{\small
History is not used. Used by Algorithm~\ref{alg:ape}.

\begin{promptbox}[title={APE-style Guided Sampling}]
You are an expert prompt engineer. Generate diverse system prompts for an AI assistant to help it perform well on a specific task.

\medskip

Task description:\\
\tvar{task\_context}

Generate exactly \tvar{q} diverse system prompts. Each should take a different approach (e.g., step-by-step reasoning, concise instructions, structured format, direct commands, role-playing, etc.).

Return a JSON array of \tvar{q} strings, where each string is a complete system prompt.
\end{promptbox}
}

\subsubsection{OPRO}
\label{prompt:opro}

{\small
History is stratified-subsampled and sorted worst-to-best. Used by Algorithm~\ref{alg:opro}.

\begin{promptbox}[title={OPRO}]
You are an expert prompt optimizer. Analyze previous system prompts and their performance scores, then generate improved prompts.

\medskip

Task description:\\
\tvar{task\_context}

Here are previous system prompts and their scores (higher is better), sorted from worst to best:

\tvar{history\_text}

\emph{[Each entry formatted as:]}\\
\texttt{--- Prompt (Score: \tvar{score}) ---}\\
\texttt{\tvar{text}}

Analyze what makes the higher-scoring prompts better. Then generate exactly \tvar{q} new system prompts that should score even higher.

Return a JSON array of \tvar{q} strings.
\end{promptbox}
}

\subsubsection{PromptBreeder Mutation}
\label{prompt:pb-mutation}

{\small
One of three mutation instructions is selected at random per offspring. Used by Algorithm~\ref{alg:pb}.

\begin{promptbox}[title={PromptBreeder Mutation}]
You are an expert prompt engineer. Modify system prompts to improve their effectiveness.

\medskip

Task description:\\
\tvar{task\_context}

Instruction: \tvar{instruction}

Original system prompt:\\
\tvar{parent\_prompt}

Output ONLY the modified system prompt, no commentary.

\medskip
\emph{The \tvar{instruction} variable is one of three mutation types (selected uniformly at random):}
\begin{enumerate}
\item \textbf{rewrite\_clearer}: ``Rewrite the following system prompt to be clearer and more precise. Keep the core instructions but improve clarity.''
\item \textbf{explicit\_reasoning}: ``Modify the following system prompt to make reasoning steps more explicit. Add instructions for step-by-step thinking.''
\item \textbf{concise\_constraints}: ``Make the following system prompt more concise. Remove redundancy while preserving all important constraints.''
\end{enumerate}
\end{promptbox}
}

\subsubsection{PromptBreeder Recombination}
\label{prompt:pb-recombination}

{\small
Used for the last offspring ($j = q$) in Algorithm~\ref{alg:pb}.

\begin{promptbox}[title={PromptBreeder Recombination}]
You are an expert prompt engineer. Combine the best aspects of two system prompts into a single improved prompt.

\medskip

Task description:\\
\tvar{task\_context}

Parent prompt 1:\\
\tvar{parent1}

Parent prompt 2:\\
\tvar{parent2}

Create a new system prompt that combines the best aspects of both parents. Output ONLY the new prompt, no commentary.
\end{promptbox}
}

\subsubsection{TextGrad-style Black-Box Refinement}
\label{prompt:textgrad}

{\small
Explicit 3-step chain-of-thought before generating variants. Used by Algorithm~\ref{alg:tg}. \tvar{best\_prompt} and \tvar{best\_score} are the global best from the full history, not just the subsampled trajectory.

\begin{promptbox}[title={TextGrad-style Black-Box Refinement}]
You are an expert prompt optimizer. Given a trajectory of system prompts and their performance scores, analyze what makes some perform better than others, then critique the current best prompt and generate improved variants.

\medskip

Task description:\\
\tvar{task\_context}

Trajectory of prior prompts and their scores (higher is better, sorted worst-to-best):

\tvar{history\_text}

Current best prompt (score: \tvar{best\_score}):\\
\tvar{best\_prompt}

Step 1: briefly analyze the trajectory -- what patterns separate high-scoring prompts from low-scoring ones?

Step 2: critique the current best prompt: what could be improved to get a higher score?

Step 3: generate exactly \tvar{q} improved variants based on your analysis and critique. Each variant should address a different aspect of the critique.

Return a JSON array of \tvar{q} strings, where each string is a complete improved system prompt.
\end{promptbox}
}

\subsubsection{Initial Dataset $\mathcal{D}_0$ Generation}
\label{prompt:d0}

{\small
Generated once per (task, seed) pair and shared across all methods. Used by Algorithm~\ref{alg:d0}.

\begin{promptbox}[title={Initial Dataset $\mathcal{D}_0$ Generation}]
You are an expert prompt engineer. Generate diverse system prompts for an AI assistant.

\medskip

Task description:\\
\tvar{task\_context}

Generate exactly \tvar{q} diverse system prompts that would help an AI assistant perform well on this task. Each prompt should take a different approach (e.g., step-by-step reasoning, concise instructions, structured format, etc.).

Return a JSON array of \tvar{q} strings, where each string is a complete system prompt.
\end{promptbox}

\emph{Note: This prompt is nearly identical to the APE prompt (Section~\ref{prompt:ape}) with minor differences: (1) the APE prompt says ``to help it perform well on a specific task'' while this one says ``for an AI assistant''; (2) the APE prompt includes ``direct commands, role-playing'' among approach suggestions, which is omitted here; (3) this prompt adds ``that would help an AI assistant perform well on this task'' after ``diverse system prompts'' and uses ``Each prompt should'' instead of ``Each should.''}
}

\end{document}